\documentclass[compsoc]{IEEEtran}

\usepackage[utf8]{inputenc}

\usepackage[pdftex]{graphicx}
\usepackage{amsmath,amssymb}
\usepackage{url}
\usepackage{path}
\usepackage{color}
\usepackage[nocompress]{cite}
\usepackage{booktabs, multicol, multirow}
\usepackage{subfigure}
\usepackage{longtable}

\newtheorem{theorem}{Theorem}[section]
\newtheorem{lemma}[theorem]{Lemma}

\newtheorem{proposition}[theorem]{Proposition}

\newcommand{\x}{\mathbf{x}}

\newenvironment{proof}[1][Proof]{\begin{trivlist}
\item[\hskip \labelsep {\bfseries #1}]}{\end{trivlist}}

\begin{document}
\title{Instance Selection Improves Geometric Mean Accuracy: A Study on Imbalanced Data Classification}

\author{Ludmila~I.~Kuncheva,
        Álvar~Arnaiz-González,
        José-Francisco~Díez-Pastor,
        and~Iain~A.~D.~Gunn
\IEEEcompsocitemizethanks{
\IEEEcompsocthanksitem L.~Kuncheva is with the School of Computer Science, Bangor University, Dean Street, Bangor, Gwynedd, Wales LL57 2NJ, UK (email: l.i.kuncheva@bangor.ac.uk).
\IEEEcompsocthanksitem \'A.~Arnaiz-Gonz\'alez and J.~D\'iez-Pastor are with the Escuela Polit\'ecnica Superior, Universidad de Burgos, Avda. de Cantabria s/n, 09006 Burgos, Spain (email: \{alvarag,jfdpastor\}@ubu.es).
\IEEEcompsocthanksitem I.~Gunn is with the Department of Computer Science, Middlesex University, London NW4 4BT (email: i.gunn@mdx.ac.uk).
}
}

\IEEEtitleabstractindextext{%
\begin{abstract}
A natural way of handling imbalanced data is to attempt to equalise the class frequencies and train the classifier of choice on balanced data. For two-class imbalanced problems, the classification success is typically measured by the geometric mean (GM) of the true positive and true negative rates. Here we prove that GM can be improved upon by instance selection, and give the theoretical conditions for such an improvement. We demonstrate that GM is non-monotonic with respect to the number of retained instances, which discourages systematic instance selection. We also show that balancing the distribution frequencies is inferior to a direct maximisation of GM. To verify our theoretical findings, we carried out an experimental study of 12 instance selection  methods for imbalanced data, using 66 standard benchmark data sets. The results reveal possible room for new instance selection methods for imbalanced data.
\end{abstract}

}

\maketitle

\IEEEdisplaynontitleabstractindextext

\IEEEraisesectionheading{\section{Introduction}
\label{introduction}}
Class imbalance arises when the number of examples\footnote{We will use the terms `example', `instance', `object' and `prototype' interchangeably, meaning a data point in the feature space of interest, e.g., $\mathbf{x}\in \mathbb{R}^n$.} belonging to one class is much greater than the number of examples belonging to another~\cite{chawla2004}. Real-life data abounds with problems of this type, from fields including bioinformatics~\cite{Pedrajas2012}, medicine~\cite{krawczyk2016,eskildsen2014}, security~\cite{cieslak2006,tesfahun2013,dal2014, phua2004}, finance~\cite{sanz2015}, software development\cite{drown2009,seiffert2009,sun2012}, and satellite imaging~\cite{zheng2015}. The class of interest is usually the minority class, e.g., fraudulent transactions, mammogram lesions, credit defaults. Handling imbalanced data sets is difficult because traditional classifiers are designed to maximise a global measure of accuracy, which often results in the minority class being ignored~\cite{Visa2005}. 

The existing methods for classification of imbalanced data can be categorised as follows~\cite{Galar2012}:

\begin{itemize}
\item The {\em algorithm-level} category includes methods which are specially designed or modified to handle imbalanced data.

\item The {\em  data-level} category includes methods which transform the data into more balanced classes so that standard classification algorithms can then be used.

\item The {\em cost-sensitive} methods lie between the first two categories. Such methods are designed by introducing different misclassification costs for the classes and modifying the training algorithm accordingly. 

\item The final category, classifier ensembles, draws upon all three previous categories. Classifier ensembles have been widely acclaimed for their ability to work with imbalanced data~\cite{diez2015div}.
\end{itemize}

The underlying assumption in all categories is that classifier training algorithms work better when the classes are balanced, especially when the performance is measured by the geometric mean (GM) of the true positive rate and the true negative rate. To balance the classes, the data-level category resorts to two alternative approaches: undersampling the majority class and oversampling the minority class. Here we are interested in the effect of undersampling on GM. Many of the undersampling methods in the previous studies are based on instance selection methods (prototype selection, editing) for the nearest neighbour classifier (1-NN)~\cite{Fix52,Cover67}. This classifier has justly received recognition for its simplicity, accuracy and interpretability, and has been listed among the top ten algorithms in data mining~\cite{Wu08}. Perhaps more than one hundred data editing methods have been proposed for 1-NN since its conception, and continue to be proposed to this day~\cite{DROP3Wilson2000,Dasarathy90,garciaSalvador2012,Triguero12}. These methods aim at reducing the reference set for the 1-NN classifier without adversely affecting its classification accuracy. We find it curious that no such methods have yet been developed to maximise GM. 
The reduction of the size of the majority class is typically done with a view to balancing the prior probabilities for the classes, rather than selecting a reference set maximising GM. 

The classifiers trained on the reduced set are often decision trees~\cite{batista2004}, ensembles of decision trees~\cite{Galar2012,galar2013}, or the support vector machine classifier (SVM)~\cite{Akbani2004,Batuwita2010}, among others. Recent studies have shown, however, that prototype selection, generation or replacement followed by 1-NN classification rivals the more intricate state-of-the-art classification methods for imbalanced data~\cite{mani2003,barandela2003a,lopez2014,saez2016}. 

We note that the GM of the classification produced by 1-NN will not in general be the same as the GM of another classifier which uses the same reduced reference set. Nonetheless, we believe that the data-manipulation heuristics are crucial for the success of classification of imbalanced problems, and the 1-NN GM will offer valuable insights about such heuristics.

This paper offers a theoretical perspective on instance selection for imbalanced data in relation to the geometric mean (GM) as the performance measure. We prove the following facts, which, we believe, have not been formally proven thus far:
\begin{enumerate}
\item The GM value may increase  when instance selection is performed before classification, compared with a baseline in which classification is performed on the raw data. Section~\ref{gm_improve} contains the proof, and the conditions for the improvement. 
\item The maximum GM for a given data set is not necessarily a monotonic or a convex function of the number of retained instances.  This is demonstrated by an example in Section~\ref{gm_mono}.
\item Equalising the class frequencies (prior probabilities) is not guaranteed to lead to the optimal GM. The proof is given in Section~\ref{gm_priors}.
\item Assuming equal prior probabilities for the classes, the Bayes classifier which maximises the classification accuracy does not necessarily maximise GM. We give a proof using a counter example in Section~\ref{gm_bayes}.  
\end{enumerate}

Experiments illustrating our theoretical findings are presented in Section~\ref{experiment}. We compare 12 instance selection methods for imbalanced data on 66 benchmark data sets. The experiment offers further insights into the properties of instance selection methods.

\section{The Geometric Mean (GM) as a measure of classification performance}
\label{tgmmfic}

Consider a two-class problem where the class of interest is chosen to be the ``positive'' class ($\omega_+$) and the other class, typically larger, is the ``negative'' class ($\omega_-$). The confusion matrix for a given classifier which assigns the respective labels $+$ and $-$ is:

\begin{center}
\rule{2.7 cm}{0 cm}Assigned labels\\
\smallskip
True labels
\begin{tabular}{rccc}
&& $+$ & $-$\\
\cline{3-4}
&$+$&$A$&$B$\\
\cline{3-4}
&$-$&$C$&$D$\\
\cline{3-4}
&&&\\
\end{tabular}
\end{center}
\noindent
where $A$, $B$, $C$, and $D$ are numbers of instances from some testing data with $N_{\mbox{test}}$ instances ($N_{\mbox{test}}=A+B+C+D$). 
The GM measure is the geometric mean of the True Positive Rate ($\mathit{TPR} = A/(A+B)$) and the True Negative Rate ($\mathit{TNR} = D/(C+D)$):
\[
\mathit{GM} = \sqrt{\textit{TPR}\times \textit{TNR}} = \sqrt{\frac{AD}{(A+B)(C+D)}} 
\]

\subsection{Improvement on GM by instance selection}
\label{gm_improve}
Here we prove that, given a labelled data set, GM for the nearest neighbour classifier may increase when instances are removed from the reference set.

We can define a nearest neighbour classifier (1-NN) by the labelled reference set it uses. Equivalently, we can define it by the set of labelled Voronoi cells, $\mathcal{V}$, corresponding to the reference instances. This set is drawn from a collection of possible such sets, $\mathbb{V}$. For example, given a training data set with $R$ instances, we may choose to look for a subset of $M$ instances for use as the 1-NN reference set. This choice of size will define $\mathbb{V}$ with ${R \choose M}$ elements. For given $\mathcal{V} \in \mathbb{V}$, denote by $V_+$ the set of Voronoi cells with label $\omega_+$, and by $V_-$, the set of Voronoi cells with label $\omega_-$, so that $\mathcal{V} = V_+\cup V_-$.
Given knowledge of the class-conditional pdfs $p(\mathbf{x}|\omega_+)$ and $p(\mathbf{x}|\omega_-)$, the GM-optimal 1-NN classifier $\mathcal{V}^*$ using $M$ prototypes from the given data set is the one which maximises the asymptotic value of GM (that is, the limit of the GM of the classifier on test data, as the number of elements of the test data set goes to infinity):
\begin{equation}
\mathcal{V}^* = \arg\max_{\mathcal{V}\in\mathbb{V}}
\left(\int_{V_+}p(\mathbf{x}|\omega_+)dx\;
\int_{V_-}p(\mathbf{x}|\omega_-)dx\right).
\label{prodint}
\end{equation}

We shall demonstrate that it is possible to improve the {\em asymptotic} GM value by excluding prototypes from the reference set, and will give a condition under which removal of a given prototype is desirable. We build the proof in three stages. First, the Cell-inclusion lemma~\ref{CellInclusion} proves that if a prototype is removed from the reference set, its Voronoi cell is ``absorbed'' by the remaining Vornoi cells such that each remaining cell is  either unchanged or is expanded. Next, the Cell expansion lemma~\ref{CellExpansion} proves that the cells which expand after removal of a prototype are those which shared a non-trivial border with the cell of the removed prototype. Finally, proposition~\ref{propreduction} states the condition on the probability distributions over the 
removed cell which guarantees that the value of GM after removing the prototype is greater than the value using the original prototype set ${\cal V}$.

\begin{lemma} [Cell inclusion]
 \label{CellInclusion}
 Let $\mathcal{V}=\{V_1,\ldots,V_N\}$ be the set of Voronoi cells for the data set  $X=\{\mathbf{x}_1,\ldots,\mathbf{x}_N\} \subset \mathbb{R}^n$. Let $V_j\in \mathcal{V}$ be the Voronoi cell of $\mathbf{x}_j\in X$. Suppose that $\mathbf{x}_i$ is removed from $X$, $i\neq j$. Let $\mathcal{V}'$ be the set of Voronoi cells for the set $X'=X\setminus \{\mathbf{x}_i\}$, and $V_j' \in \mathcal{V}'$ be the new Voronoi cell of $\mathbf{x}_j\in X'$. Then
 \[
 V_j\subseteq V'_j\;.
 \]  
\end{lemma}

\begin{proof}
$X' \subset X$.  Therefore, if a point is closer to $\mathbf{x}_j$ than to any other point in $X$, then it is closer to $\mathbf{x}_j$ than to any other point in $X'$.  So each element of $V_j$ is an element of $V_j'$.
\end{proof}

If $\mathbf{x}_i$ was not one of neighbours of $\mathbf{x}_j$ defining its Voronoi cell $V_j$, then $V'_j$ and $V_j$ are identical. If $\mathbf{x}_i$ was one of $\mathbf{x}_j$'s neighbours, then $V'_j$ will be a proper superset of $V_j$,  as we will now prove.

\begin{lemma} [Cell expansion]
	\label{CellExpansion}
    Let $\mathbf{x}_i$, $\mathbf{x}_j$, $X$, $X'$, $V_j$, $V_j'$ be as in Lemma \ref{CellInclusion}, with the Voronoi cells being defined using Euclidean distance.  Let $L$ be the Lebesgue measure on $\mathbb{R}^{n-1}$.  Let 
    \[
    B = \{\mathbf{x}: \|\mathbf{x} - \mathbf{x}_i\| = \|\mathbf{x} - \mathbf{x}_j\| \leq \|\mathbf{x} - \mathbf{x}_k\| \;\forall k \neq i,j\}
    \]
    be the set of points constituting the common border of the Voronoi cells of $\mathbf{x}_i$ and $\mathbf{x}_j$. Then
    \begin{equation}
        L(B)>0 \Rightarrow V_j \subset V_j'.
    \end{equation}
    That is, all neighbouring cells whose common border with $V_i$ is a true facet will expand when $\mathbf{x}_i$ is removed to form $X'$.
\end{lemma}

\begin{proof}
    Define an extremal border point with respect to $k$ to be a point $\mathbf{x}$ such that
    \begin{equation}
        \|\mathbf{x} - \mathbf{x}_i\| = \|\mathbf{x} - \mathbf{x}_j\| = \|\mathbf{x} - \mathbf{x}_k\|
    \end{equation}
    for some $k \neq i,j$.  Then for each $k \neq i,j$, the extremal border points with respect to that $k$, if any exist, lie on the intersection of two hyperplanes of dimension $n-1$, these being the hyperplane of points equidistant from $\mathbf{x}_i$ and $\mathbf{x}_j$, and the hyperplane of points equidistant from $\mathbf{x}_i$ and $\mathbf{x}_k$.  The extremal border points w.r.t. $k$ are therefore a subset of a ($n-2$)-dimensional space.  They are therefore a set of measure zero, using the Lebesgue measure on $\mathbb{R}^{n-1}$.  The union of all $N-2$ such sets is therefore also a set of measure zero.
    
    Therefore, if $L(B)>0$, then $B$ must contain points which are not extremal border points w.r.t. any $k$.  For these points, and, by the continuity of the Euclidean distance function, for nearby points in $V_i \setminus B$, $\|\mathbf{x} - \mathbf{x}_j\| < \|\mathbf{x} - \mathbf{x}_k\| \forall k \neq i,j$.
    
    Thus, $\exists \mathbf{x} \in V_i \setminus B$ such that $\mathbf{x} \in V_j'$.  This, combined with the result $V_j \subseteq V_j'$ from Lemma~\ref{CellInclusion} above, implies the result $V_j \subset V_j'$. 
\end{proof}

These lemmata are illustrated for the two-dimensional case in Figure~\ref{lemmafig}. The left subplot shows the original Voronoi diagram $\mathcal{V}$ of a set $X$. The cell to be removed is highlighted. The right subplot shows the expanded cells in  $\mathcal{V}'$ resulting from the removal. 

Suppose that an instance $\mathbf{x}_i$ from the positive class is removed from $X$. Suppose that $\mathcal{N}_x=\{\mathbf{x}_1,\ldots,\mathbf{x}_k\}$ is the set of all the neighbours of the removed point $\mathbf{x}_i$, i.e. those points having a Voronoi facet in common with $\mathbf{x}_i$. By the Cell expansion lemma, Vononoi cell corresponding to each of these neighbours will expand, that is, for each $\mathbf{x}_j\in \mathcal{N}_x$, we have $V_j\subset V'_j$. Due to the strict inclusion, there exists a region $R_j\subset V'_j$ such that $R_j\not\subset V_j$. The union of all regions $R_j$ will make up $V_i$. 

\begin{figure}[htb]
\centering
\includegraphics[width=1\linewidth]{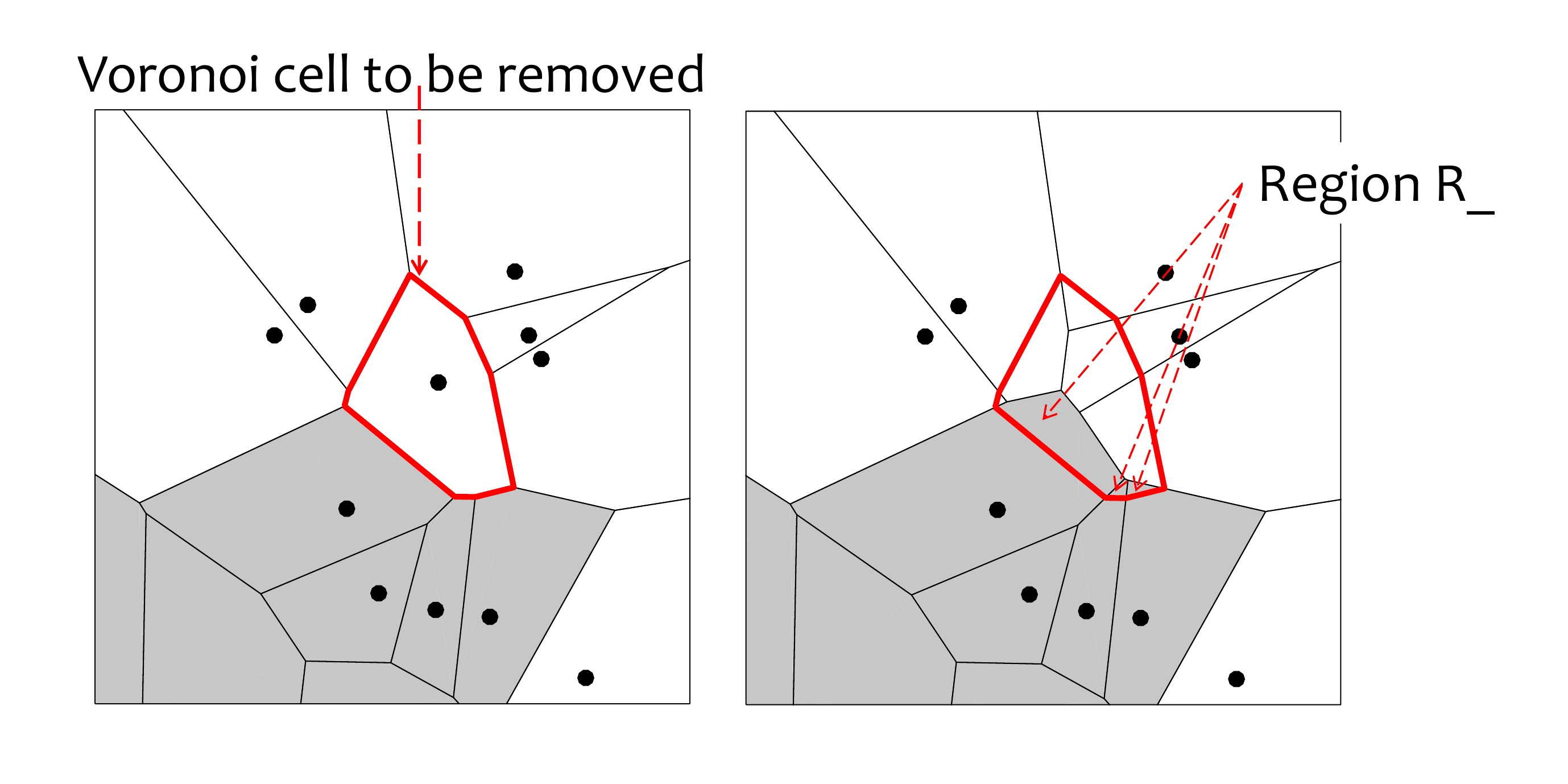}
\caption{Illustration of the region $R_-$ when the outlined Voronoi cell from class positive (white) is removed.}
\label{lemmafig}
\end{figure}

Let the labels of the points in $\mathcal{N}_x$ be $\{y_1,\ldots,y_k\}$, $y_j\in\{\omega_+,\omega_-\}$. If there is an $\mathbf{x}_j \in \mathcal{N}_x$ with $y_j = \omega_-$, then the negative classification region will {\em gain} volume through expanding the Voronoi cells in $N_x$ whose label is $\omega_-$. The addition to the negative region is
\[
R_- = \bigcup_{y_j=\omega_-} R_j.
\]
The positive classification  region will {\em lose} the same volume. The right subplot in Figure~\ref{lemmafig} illustrates $R_-$ in the 2D space. 

While the volume removed from the classification region of one class is equal to that added to the classification region of the other class, the associated probability masses will not in general be equal. Let $g_-$ be the gain of true negative rate ($\textit{TNR}$) acquired from region $R_-$.
\begin{equation}
\label{gminus}
g_- = \int_{R_-} p(\mathbf{x}|\omega_-) dx\;.
\end{equation}
Similarly, let $l_+$ be the loss of true positive rate ($\textit{TPR}$) incurred by removing $R_-$ from the positive classification region.
\begin{equation}
\label{lplus}
l_+ = \int_{R_-} p(\mathbf{x}|\omega_+) dx\;.
\end{equation}
Let $\textit{TPR}$ and $\textit{TNR}$ be the rates calculated on $\mathcal{V}$ (before the removal of $\mathbf{x}_i$).

\begin{proposition} [GM improvement through data reduction]
\label{propreduction}
Let $\mathcal{V}, \mathcal{V}',$ and $X$ be as in Lemma~\ref{CellInclusion}, and suppose further that the instance $\mathbf{x}_i$ to be removed is of the positive class. Let $\textit{TPR}$ and $\textit{TNR}$ be respectively the asymptotic True Positive Rate and True Negative Rate of a 1-NN classifier defined by $\mathcal{V}$.  Let $l_+$ and $g_-$ be respectively the loss in $\textit{TPR}$ and gain in $\textit{TNR}$ associated with the removal $\mathbf{x}_i$, defined in equations \ref{gminus} and \ref{lplus} above. Then if 
\begin{equation}
\frac{(\textit{TPR}-l_+)(\textit{TNR}+g_-)}{\textit{TPR}\times\textit{TNR}} > 1\;,
\label{propo}
\end{equation}
then the geometric mean criterion (GM) of the 1-NN classifier defined by $\mathcal{V}'$ is greater than that of the classifier defined by $\mathcal{V}$:
\[
GM(\mathcal{V}) < GM(\mathcal{V}')\;. 
\]
That is, the GM score of the classifier increases after the removal of $\mathbf{x}_i$.
\end{proposition}
\begin{proof}
Consider the difference between the squares of the two geometric means
\[
\mathit{GM}(\mathcal{V})^2 - \mathit{GM}(\mathcal{V}')^2 
\rule{9cm}{0 cm} 
\]
\begin{eqnarray}
&=&\int_{V_+}p(\mathbf{x}|\omega_+)dx
\int_{V_-}p(\mathbf{x}|\omega_-)dx\nonumber \\
&&-\int_{V_+\setminus R_-}p(\mathbf{x}|\omega_+)dx
\int_{V_-\cup R_-}p(\mathbf{x}|\omega_-)dx\nonumber\\
&=& \textit{TPR}\times \textit{TNR}- \left(\int_{V_+}p(\mathbf{x}|\omega_+)dx-\int_{R_-}p(\mathbf{x}|\omega_+)dx\right)\nonumber\\
&&\hspace{1.5 cm}\times \left(\int_{V_-}p(\mathbf{x}|\omega_-)dx + \int_{R_-}p(\mathbf{x}|\omega_-)dx \right)\nonumber\\
&=&\textit{TPR}\times \textit{TNR}-(\textit{TPR}-l_+)(\textit{TNR}+g_-)\;\nonumber
\end{eqnarray}
Then the inequality (\ref{propo}) implies that the RHS of the above is negative, which implies that the LHS is negative, which implies the result. 
\end{proof}

The importance of this proposition is twofold. First, it gives theoretical ground for removing prototypes from the reference set for the purpose of improving the GM criterion. Second, the result indicates that, given knowledge of the $\textit{TPR}$ and $\textit{TNR}$ of a classifier, it is enough to evaluate only the gain of true negative rate  $g_-$ and the  loss of true positive rate $l_+$ over a sub-region of the removed Voronoi cell $V_i$ in order to calculate the criterion value for the new Voronoi tessellation.

An equivalent result holds if a prototype $\mathbf{x}_i$ from the negative class is removed. In this case, $\mathit{GM}(\mathcal{V}) < \mathit{GM}(\mathcal{V}')$ when
\[
\frac{(\textit{TPR}+g_+)(\textit{TNR}-l_-)}{\textit{TPR}\times\textit{TNR}} > 1\;,
\]
where
\[
R_+ = \bigcup_{y_j=\omega_+} R_j
\]
is the added region for $\omega_+$,
\[
g_+ = \int_{R_+} p(\mathbf{x}|\omega_+) dx
\]
is the gain of probability mass for the positive class obtained from the expansion of the positive-labelled neighbours of the removed $V_i$, and
\[
l_- = \int_{R_+} p(\mathbf{x}|\omega_-) dx
\]
is the loss for the negative class.

\subsection{Sub-optimality and non-monotonicity of sequential instance selection}
\label{gm_mono}

Suppose that we remove prototypes one by one,  in such a way as to increase GM at each step. Unfortunately, this is not guaranteed to lead to an optimal reference subset, because the asymptotic GM may vary non-monotonically as prototypes are removed, giving rise to the risk of finding a local maximum. 
A comparison with the problem of feature selection may be helpful for illustrating this point. The literature on feature selection abounds with procedures for looking for an optimal feature subset through sequential forward or backward selection, floating search, genetic algorithms and more~\cite{Pudil94,Jain97a,Saeys07}. These approaches to feature selection rely on the assumption that every feature has its own individual merit in the classification context, which can be amplified by combining this feature with other features. However, with instance selection, this assumption is no longer true. The value of an instance is very strongly dependent on which other instances are in the reference set. This discourages instance selection procedures based on sequential search, in favour of randomised ones. The only way to guarantee that an optimal reference set has been selected is exhaustive search among all possible subsets of instances.

The following example demonstrates the non-monotonicity 
of GM with respect to the number of prototypes. 

The data in the example consists of $15$ two-dimensional points, shown in Figure~\ref{E3}(b), partitioned into two classes according to the ground-truth labelling of the space shown in Figure~\ref{E3}(a). The GM values are shown underneath the plots. An exhaustive search was performed, in which all possible subsets of instances were evaluated with respect to the ground truth. The best set is shown in plot (c).

\begin{figure}[htb]
\centering
\begin{tabular}{ccc}
\includegraphics[width=0.28\linewidth]{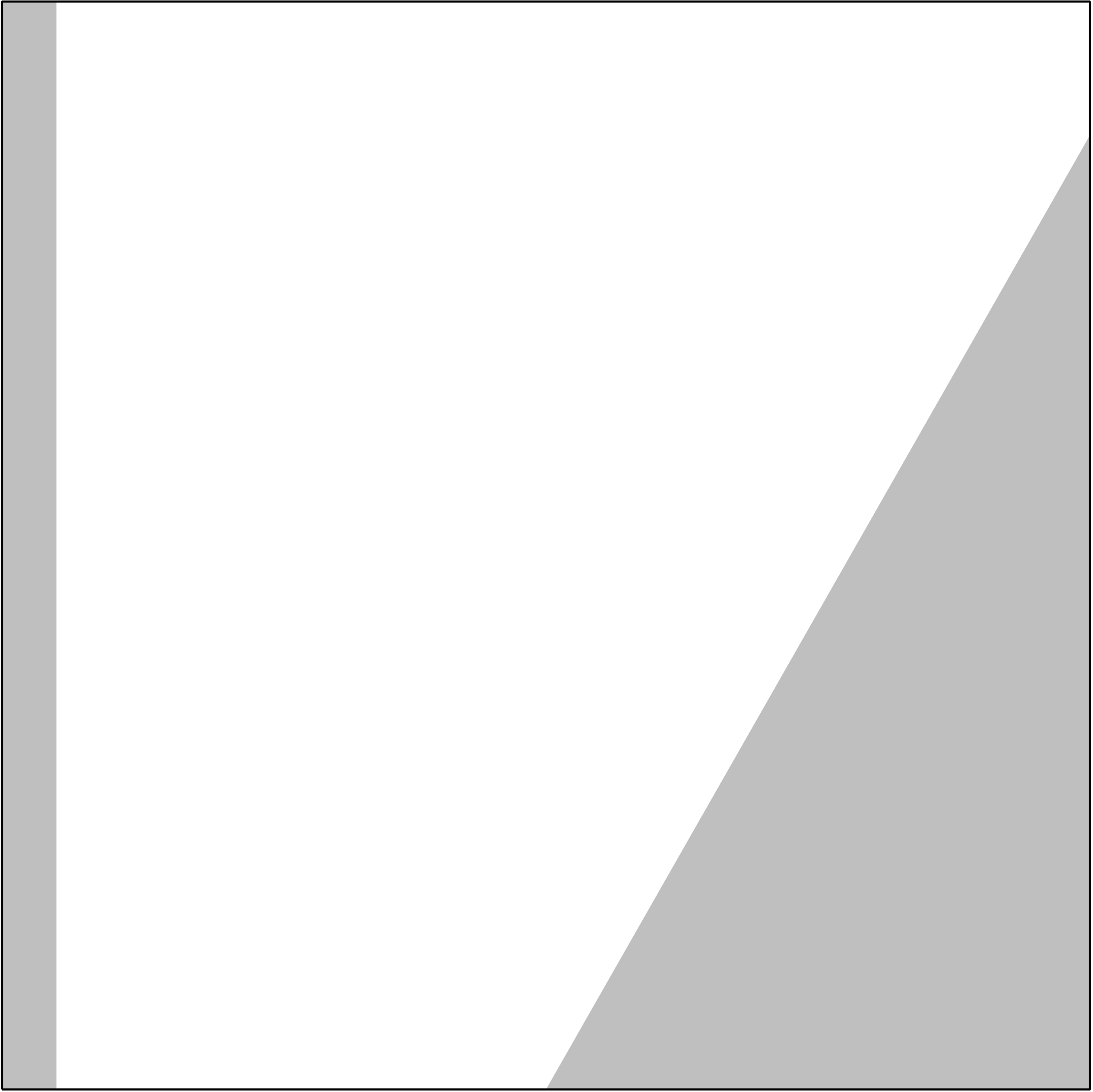}&
\includegraphics[width=0.28\linewidth]{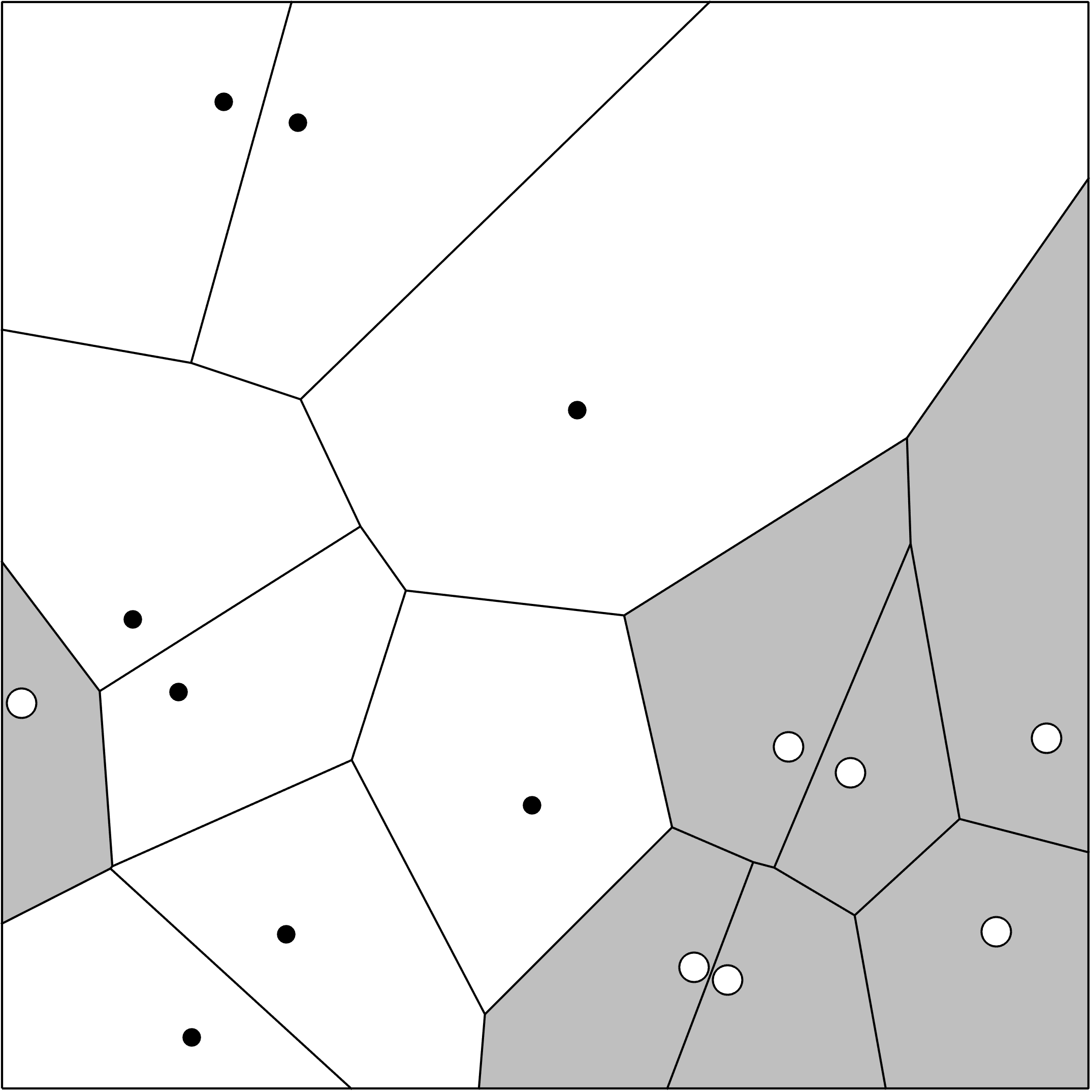}&
\includegraphics[width=0.28\linewidth]{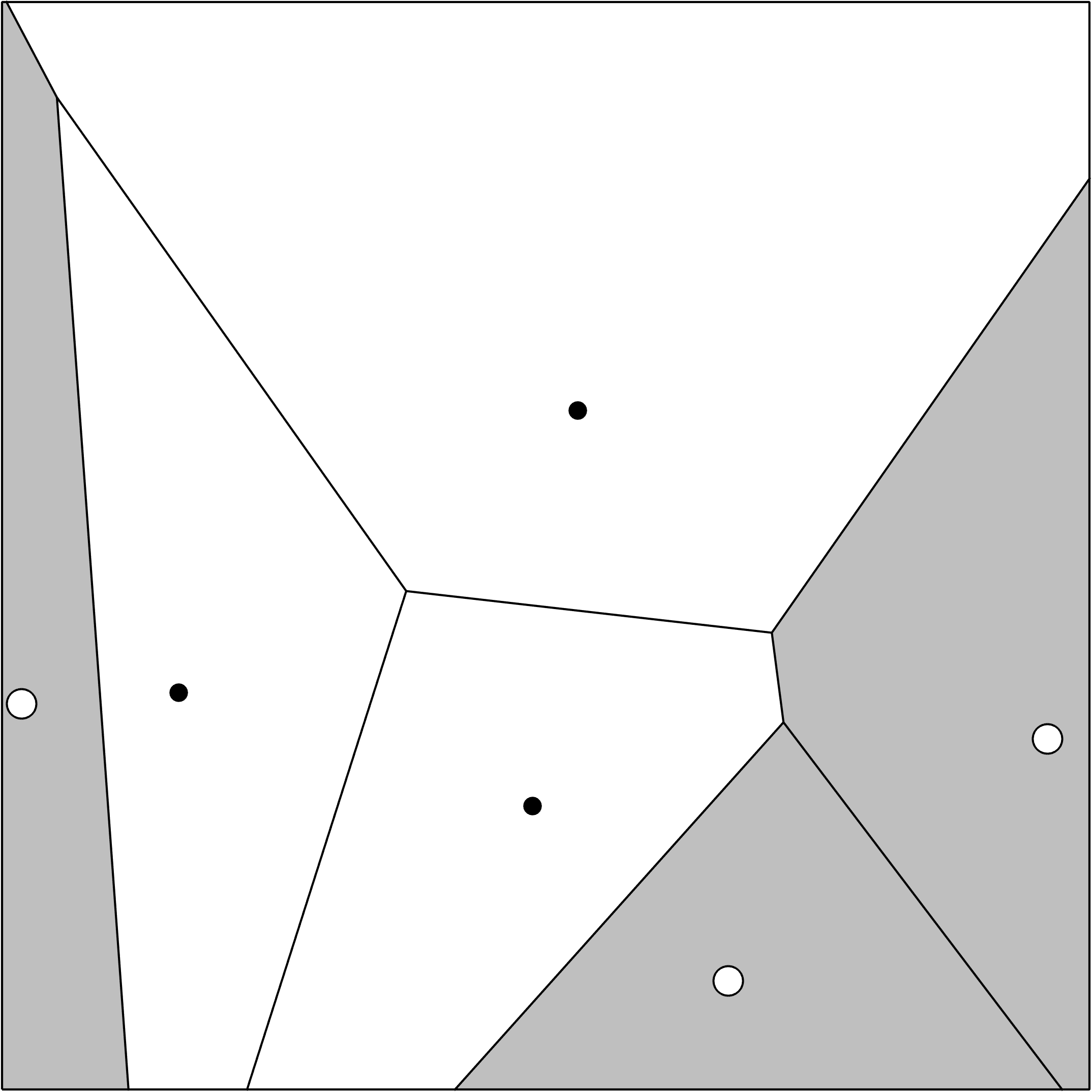}\\
\small{(a) Ground truth}&\small{(b) Data}&
\small{(c) Exhaustive}\\
$\mathit{GM} = 1.0000$&$\mathit{GM} = 0.8892$&$\mathit{GM} = 0.9549$\\
\end{tabular}
\caption{The the ground truth (a), data set (b), and the best results from applying the exhaustive search (c).}
\label{E3}
\end{figure}


Figure~\ref{Example3GM} shows the GM value of the best-performing reference set of each size.

\begin{figure}[htb]
\centering
\includegraphics[width=0.8\linewidth]{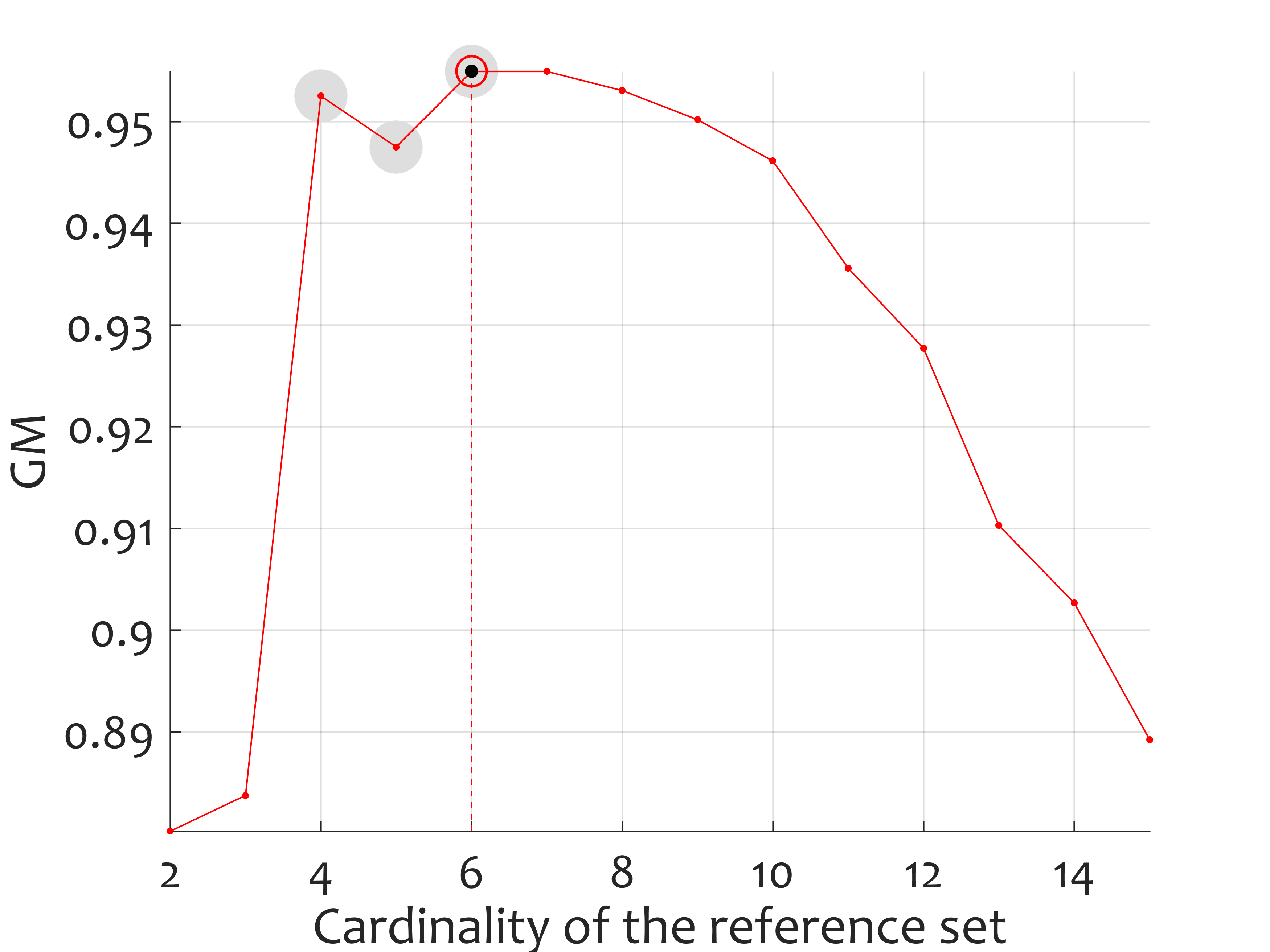}
\caption{Geometric mean (GM) value of the optimal reference set for different cardinality of the reference set for the example in Figure~\ref{E3}. The points of non-monotonicity are indicated with grey circles.}
\label{Example3GM}
\end{figure}

The example demonstrates the following points:
\begin{itemize}
    \item Given a data set, using fewer instances can be better than using all instances (an illustration of the result of Proposition~\ref{propreduction}).
    \item While GM is an approximately concave function of the number of prototypes in the reference set, the function is not necessarily monotonic in the mostly ascending or mostly descending parts.
\end{itemize}
The second point suggests that the only method which guarantees that the best reference set is selected is an exhaustive search: even the best optimisers can be foiled by such quirky behaviour of the target function.

\subsection{Equal prior probabilities do not guarantee a GM-optimal classifier}
\label{gm_priors}
It is tempting to assume that the optimal classifier is the one that assigns the class labels based on the largest class-conditional pdf (not scaled by the respective prior probabilities). Here we show that this approach does not guarantee GM-optimality.

Figure~\ref{OptimalGM} shows the class-conditional pdfs for the two classes in a one-dimensional problem:
\[
p(x|\omega_+) = \left\{\begin{array}{ll}
\frac{1}{9},&\mbox{ if } 0\leq x \leq 9\\
0,&\mbox{ elsewhere. }\end{array}\right.
\]
and
\[
p(x|\omega_-) = \left\{\begin{array}{ll}
\frac{1}{7},&\mbox{ if } 3\leq x \leq 10\\
0,&\mbox{ elsewhere. }\end{array}\right.
\]

\begin{figure}[htb]
\centering
\includegraphics[width=0.9\linewidth]{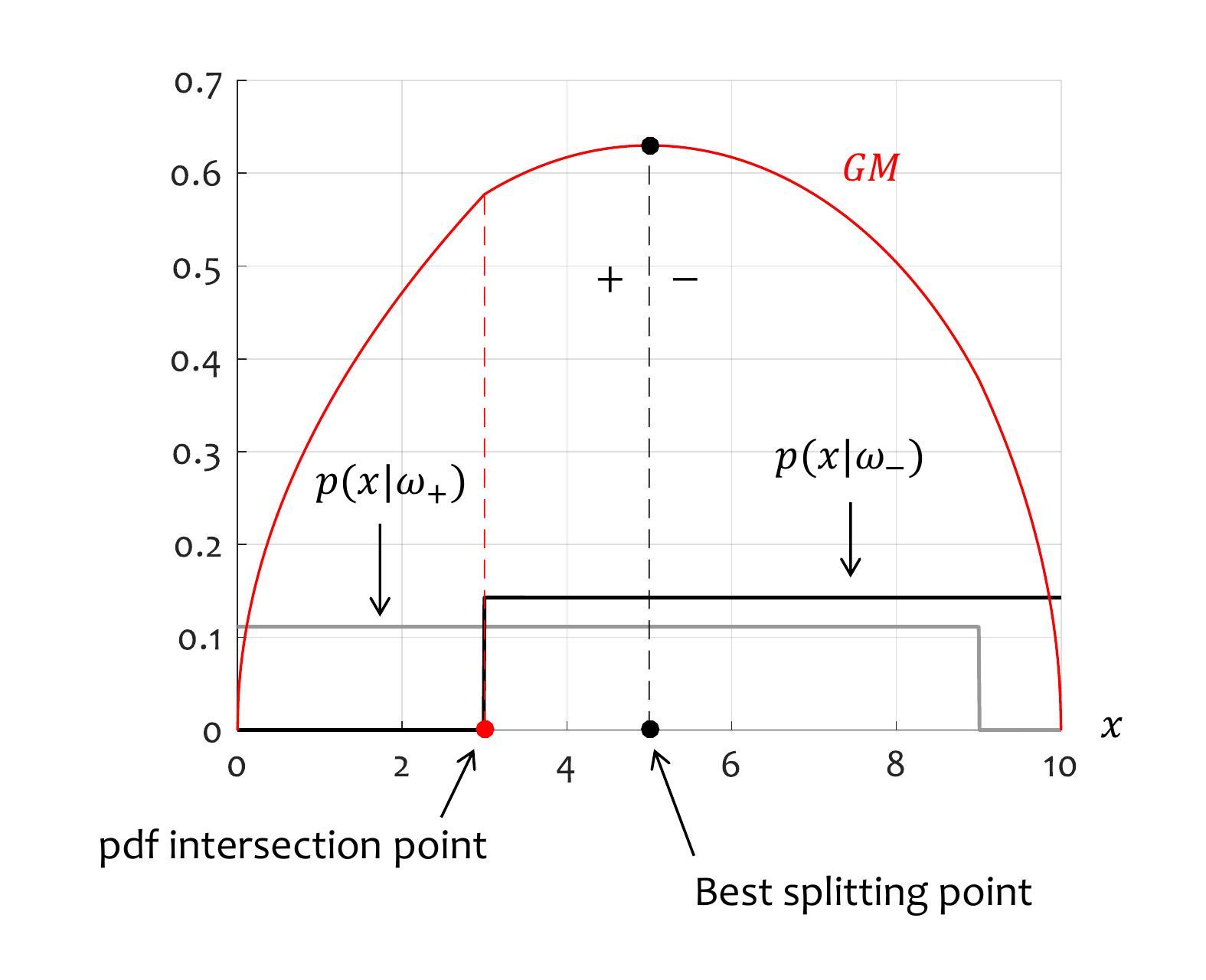}
\caption{An example for which the GM measure takes its maximum value for a decision boundary at $x=5$, and not at the boundary $x=3$ defined by the intersection of the two pdfs $p(x|\omega_+)$ and $p(x|\omega_-)$.}
\label{OptimalGM}
\end{figure}

As the pdfs are non-zero only in the interval $[0,10]$, only this interval is shown in the figure. Suppose that we slide a possible classification boundary $b$ from $0$ up to $10$ and calculate the corresponding $\mathit{TPR}(b)$, $\mathit{TNR}(b)$
\[
\mathit{TPR}(b) = \left\{\begin{array}{ll}
\frac{b}{9},&\mbox{ if } b \in [0,9]\\
1,&\mbox{ if } b \in (9,10] \end{array}\right.
\]
\[
\mathit{TNR}(b) = \left\{\begin{array}{ll}
1,&\mbox{ if } b \in [0,3)\\
\frac{10-b}{7},&\mbox{ if } b \in [3,10]
 \end{array}\right.,
\]
and, subsequently, GM: 
\[
\mathit{GM}(b) = \left\{\begin{array}{ll}
\sqrt{\frac{b}{9}},&\mbox{ if } 0\leq x \leq 3\\
\sqrt{\frac{b(10-b)}{63}},&\mbox{ if } 3 < x \leq 9\\
\sqrt{\frac{10-b}{7}},&\mbox{ if } 9 < x \leq 10
\end{array}\right.
\]

GM is shown in Figure~\ref{OptimalGM} as a function of the boundary. The point of intersection of the two pdfs is $x = 3$ (marked in the figure). The GM value for this boundary is $\mathit{GM}= \sqrt{3/9}\approx 0.5774$. There is a higher value of GM, however. It is easy to show that the maximum of GM is reached for $b=5$ and equals $\sqrt{5(10-5)/63}\approx 0.6299$. The optimal boundary is indicated in the figure (called the Best splitting point). 

This example shows that the GM-optimal classifier cannot in general be formed in the simple way that intuition might suggest -- labelling the data points according to the larger $p(\omega_k|\x)$, $k\in\{+,-\}$. 

There are cases, nonetheless, where GM is indeed maximised at the intersection of the pdfs. Two such examples are shown in Figure~\ref{MoreOptimalGM}. 

\begin{figure}[htb]
\centering
\begin{tabular}{cc}
\includegraphics[width=0.45\linewidth]{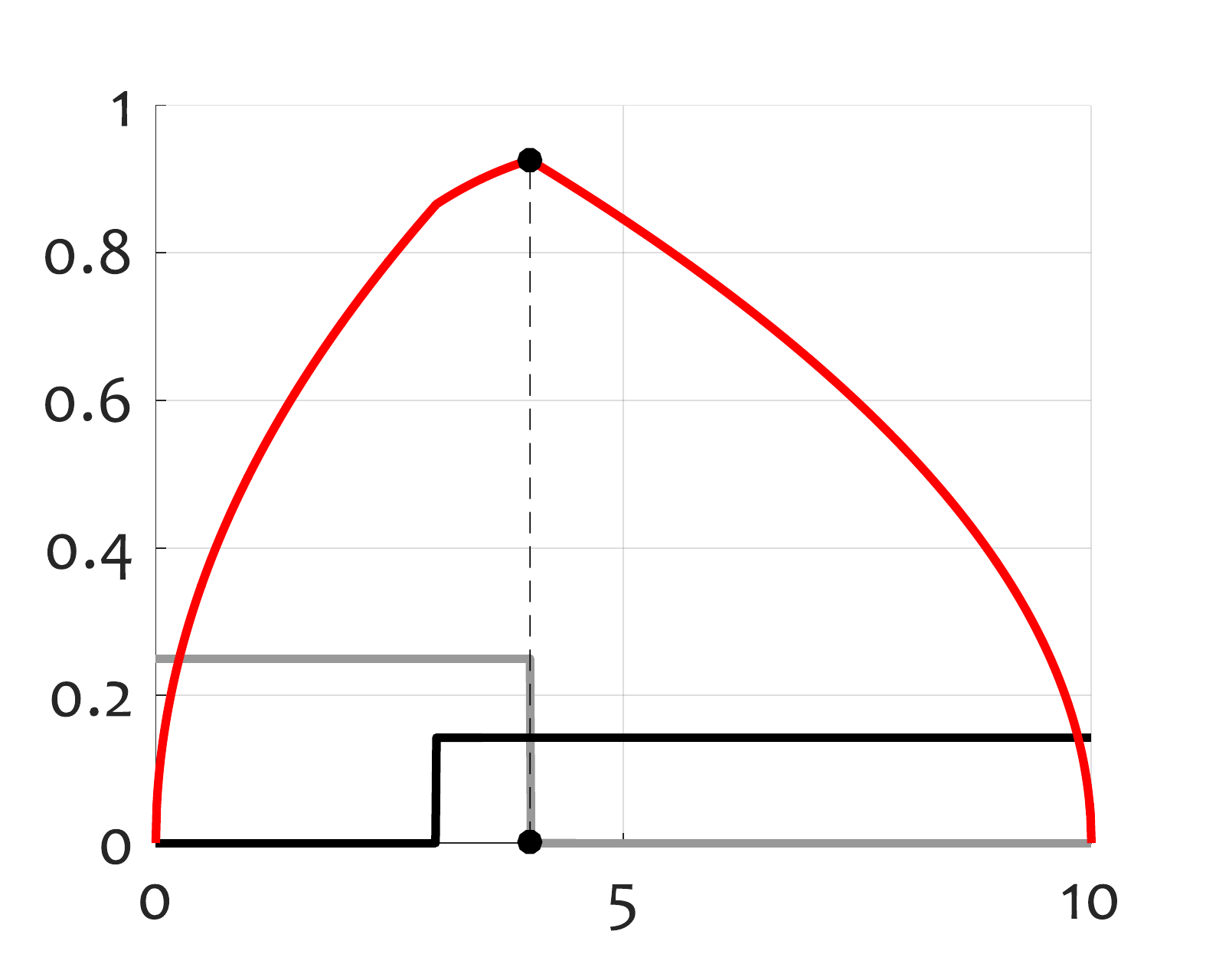}&
\includegraphics[width=0.45\linewidth]{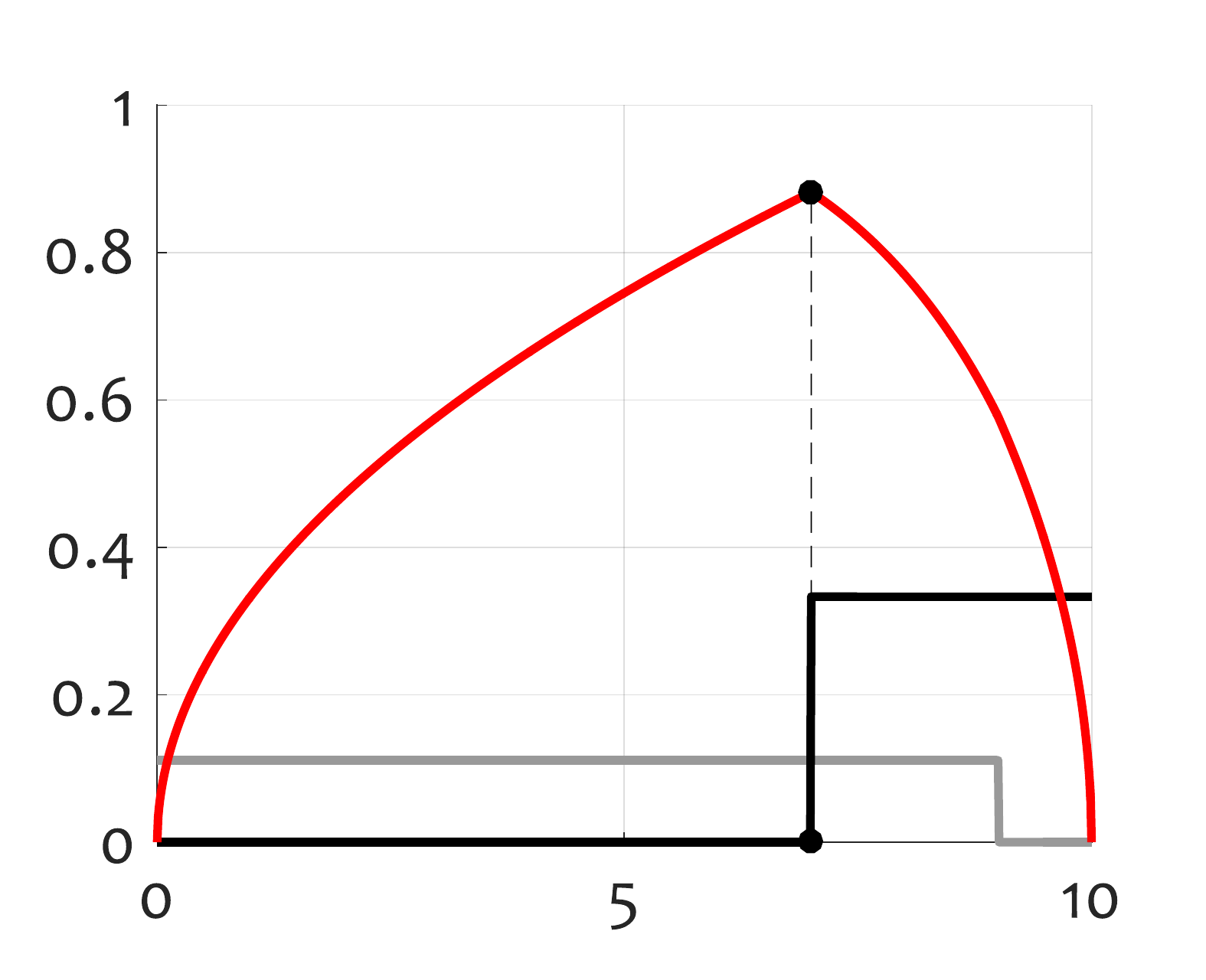}\\
\end{tabular}
\caption{Two examples where the GM (red line) does take its maximum at the intersection of the pdfs (grey line for $p(x|\omega_+)$ and black line for $p(x|\omega_-)$). }
\label{MoreOptimalGM}
\end{figure}

\subsection{Optimality and non-optimality of the Bayes classifier with respect to GM}
\label{gm_bayes}
When the class-conditional pdfs are scaled by the prior probabilities, and the intersection is chosen as the classification boundary, we arrive at the {\em Bayes classifier}. Denote this classifier by $\mathit{CB}$, for ``classical Bayes'' classifier. Consider a classifier which assigns the class labels based on the unscaled class-conditional pdfs, as in section \ref{gm_priors}, called the {\em Balanced Bayes classifier} and denoted by $\mathit{BB}$. Finally, let $\mathit{GM}^*$ be the GM-optimal classifier.

This subsection looks into the relationship between $\mathit{CB}$, $\mathit{BB}$ and $\mathit{GM}^*$. The implication of these results is practical. If GM is the chosen measure of quality for imbalanced classes, the standard algorithms for training classifiers are likely to give inferior results compared to bespoke ones.

\begin{proposition}[GM-optimality of $\mathit{CB}$ and $\mathit{BB}$] 
If the Bayes classification error for the problem is zero, the Bayes classifier and the Balanced Bayes classifier are optimal in the sense of GM  for any prior probabilities $P(\omega_+)$ and $P(\omega_-)\equiv 1- P(\omega_+)$ .
\end{proposition}
\begin{proof}
The Bayes error is zero, $\mathit{TPR}=\mathit{TNR}=1$,
giving $\mathit{GM}(\mathit{CB})=1$. Consider now the $\mathit{BB}$ classifier. The fact that the Bayes error is zero implies there is no region in which the pdfs of the two classes are both non-zero.  This implies that the classification error of the $\mathit{BB}$ classifier is also zero, with the same consequences as just outlined for the $\mathit{CB}$ classifier, giving $\mathit{GM}(\mathit{BB})=1$. The $\mathit{CB}$ and $\mathit{BB}$ classifiers have the highest possible value of GM and are therefore optimal with respect to this measure.
\end{proof}
(The $\mathit{GM}^*$, $\mathit{BB}$ and $\mathit{CB}$ classifiers have the same decision boundary in this case.)

\medskip
\begin{proposition}[Non-optimality of $\mathit{CB}$ and $\mathit{BB}$] 
\label{CBBB}
Neither the Bayes classifier $\mathit{CB}$ nor the Balanced Bayes classifier $\mathit{BB}$ are necessarily optimal with respect to GM.
\end{proposition}
\begin{proof}
The possible non-optimality of the Balanced Bayes classifier is demonstrated directly by the argument for the GM-optimal classifier in the previous subsection and illustrated in Figure~\ref{OptimalGM}.

The possible non-optimality of $\mathit{CB}$ can be argued using the same figure. When the class-conditional pdfs are multiplied by the respective prior probabilities $P(\omega_+)$ and $P(\omega_-)=1-P(\omega_+)$, we can have either of the horizontal arms of the scaled pdfs higher than the other. Then the intersection point will either stay the same at $x=3$ or change to  $x=9$. In both cases, the Bayes-optimal classifier $\mathit{CB}$ will miss the GM-optimal boundary at $x=5$.
\end{proof}

The non-optimality argument does not quantify how much the GM criterion will suffer if $\mathit{CB}$ or $\mathit{BB}$ is used instead of $\mathit{GM}^*$, nor does it give any estimate of how often optimality and non-optimality may occur in real-life problems. Therefore, based on theory, we cannot recommend that the prior probabilities should be scaled to 0.5 by balancing the samples from the classes, nor can we recommend that either $\mathit{CB}$ or $\mathit{BB}$ should be used instead of classifier trained explicitly with GM as the criterion. The practical situation, however, may not reflect this, as evidenced by a multitude of successful experiments with heuristic-based balancing methods.

The following example provides an illustration of the differences between the three classifiers. This time the classes come from a mixture of Gaussians, and the probability of error is strictly positive. The data is generated from the following distributions: 

\medskip\noindent
for the negative class (majority),
\[
p(\mathbf{x}|\omega_-)\sim N\left([0, 0]^T,I(2)\right),
\]
where $I(2)$ is the identity matrix of size 2; and for the positive class (minority),
\[
p(\mathbf{x}|\omega_+) = 0.6p_a(\mathbf{x}|\omega_+) + 0.4p_b(\mathbf{x}|\omega_+),
\]
where
\[p_a(\mathbf{x}|\omega_+)\sim N\left([-1, 1]^T,
\left[\begin{array}{rr}1&0\\0&0.3\end{array}\right]\right)
\]
and
\[p_b(\mathbf{x}|\omega_+)\sim N\left([2, -2]^T,
\left[\begin{array}{rr}0.4&0\\0&0.7\end{array}\right]\right).
\]
4,000 instances were generated from $p(\mathbf{x}|\omega_-)$ and 500 from the mixture of the positive class (300 from $p_a(\mathbf{x}|\omega_+)$ and 200 from $p_b(\mathbf{x}|\omega_+)$). Thus the prior probabilities were $\frac{8}{9}$ for the negative class and $\frac{1}{9}$ for the positive class.  The scatterplot of the data is shown in Figure~\ref{E2GM}. 

\begin{figure}[htb]
	\centering
	\includegraphics[width=0.8\linewidth]{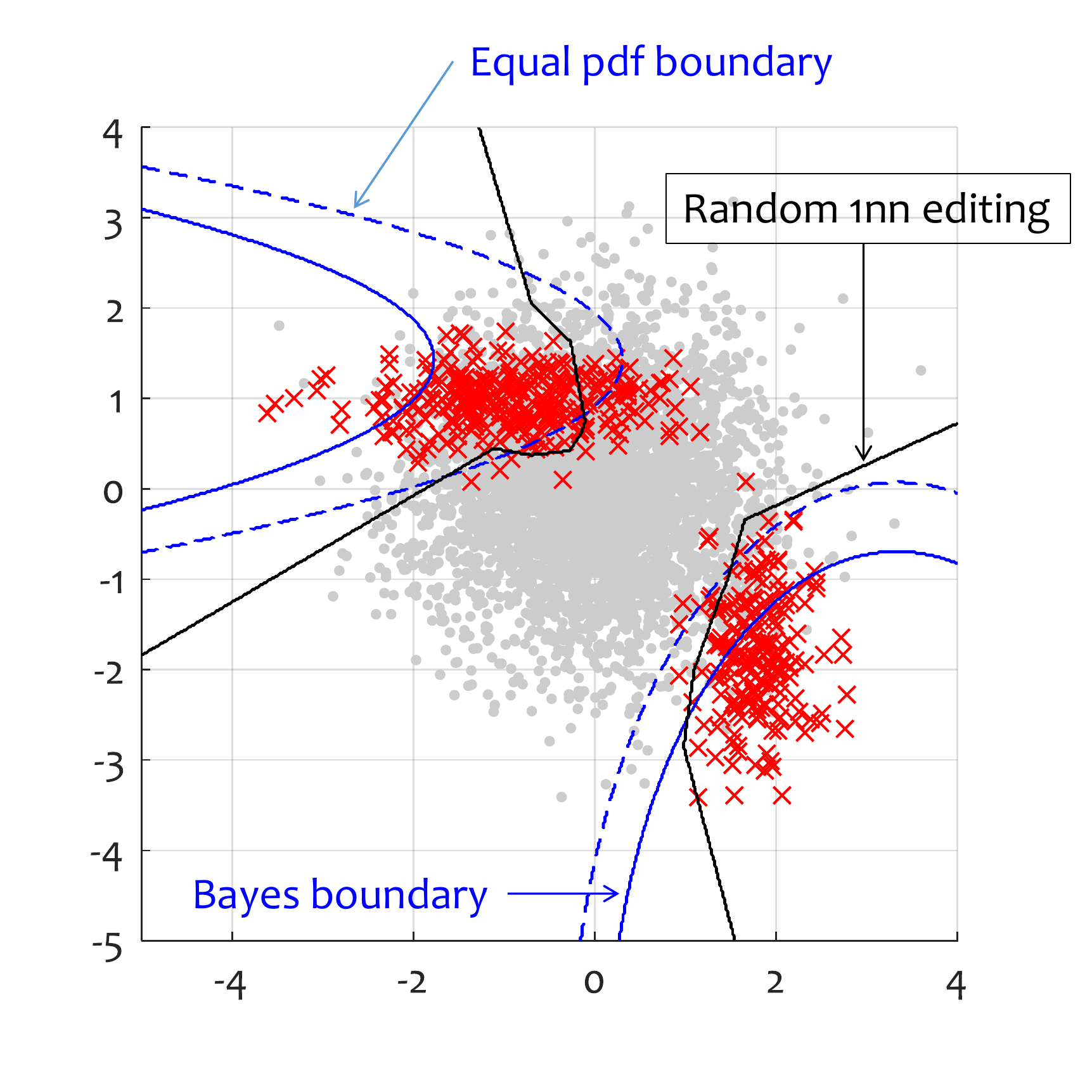}
	\caption{Scatterplot of the data for Example 2. The optimal classification boundaries for the Bayes classifier ($\mathit{CB}$,$\mathit{GM}=0.6383$) are shown with a solid blue line, and the boundaries for the Balanced Bayes classifier ($\mathit{BB}$, $\mathit{GM}=0.8322$), with a dashed blue line. The black line shows the classification regions for a 1NN with a reference set with 25 instances found through random editing ($\mathit{GM}=0.8330$).}
	\label{E2GM}
\end{figure}

Shown in Figure~\ref{E2GM} are also the 
classification boundaries for the Bayes classifier ($\mathit{CB}$, the solid blue line) and the Balanced Bayes classifier ($\mathit{BB}$, the dashed blue line). The GM estimates were calculated on a separate set of size 9,000 generated from the distribution of the problem. The GM value for $\mathit{CB}$ is 0.41, which is low, as expected. The GM for $\mathit{BB}$ is 0.70. This reflects the expected situation in real-life problems, giving support to the heuristic that balancing the classes (when the Bayes error is non-zero) is likely to lead to substantial improvement. 


Next we created an example to demonstrate that simple random editing (RE) may give a better GM compared to both the Bayes-optimal classifier and the Balanced Bayes classifier. We sampled 10,000 reference sets of cardinality 25 from the original data (4,500 instances), and chose the best set according to the training value of the RE GM. We chose the example so as to illustrate the possibility that the testing values of RE GM are better than both rival values: CB GM and BB GM. The classification regions for the random editing method are delineated in Figure~\ref{E2GM}. 



To summarise, the second example shows that:
\begin{itemize}
	\item Balancing of the classes in the case of non-separable classes (non-zero Bayes error) may be significantly better than using a classifier trained to minimise the classification error.
	\item 1-NN using using simple random editing (RE) may be better than both $\mathit{CB}$ and $\mathit{BB}$, even though the latter are formed using full knowledge of the underlying pdfs. This suggests that explicit maximisation of GM is likely to fare better than balancing the class proportions.
\end{itemize}

\section{Experiment} 
\label{experiment}

An experiment was designed to explore the implications of the theoretical findings in Section~\ref{tgmmfic} with respect to the most acclaimed methods for instance selection for imbalanced data. 

We aim to show that: 
\begin{enumerate}
  \item Instance selection generally improves GM, as suggested to be possible by Proposition~\ref{propreduction}, and the  example in Figure~\ref{E3}. As an example, we contrast 1-NN with the edited 1-NN. (Section~\ref{gm_improve}) 
  \item Random prototype selection may be better than systematic (non-random) selection. (Section~\ref{gm_mono})   
\item Optimising GM explicitly may be more successful than balancing the prior probabilities alone or not accounting for the imbalance at all. (Section~\ref{gm_priors})
\end{enumerate}

In the course of the experimental study, we also discovered that classifier ensemble methods for instance selection for imbalanced data are better than single instance selection methods.
It will be seen that there are significant unexplored possibilities for designing potentially successful methods for imbalanced data.      

\subsection{Experimental protocol}

We used two-fold cross-validation repeated 5 times~\cite{dietterich1998,diez2015rb}. Each data set is split into two sets of equal size; one is used for training and the other for testing. After that, the roles are reversed. The cross-validation folds are drawn using stratified sampling so that the imbalance ratio is mirrored in the training and testing parts. We decided against the common choice of using 10-fold cross-validation, because splitting of the imbalanced data sets into too many folds may produce folds with an inadequately small number of objects from the minority (positive) class.

\subsection{Data sets}

We used data sets from the KEEL data collection~\cite{Alcala2011}. This repository contains 66 binary imbalanced data sets\footnote{Available at \url{http://sci2s.ugr.es/keel/imbalanced.php}.}. These data sets are not all completely independent of each other. Several of them are variants of an original data set. For example, the ``yeast'' data set gives rise to 14 variants where the 10 classes are grouped in different ways to make the 14 imbalanced two-class problems. Table~\ref{tab:data-keel} shows the characteristics of this data set collection.

\begin{table*}[tb]
	\caption[Characteristics of the data sets from the KEEL collection.]{Characteristics of the data sets from the KEEL collection. 
		Column \#E shows the number of examples in the data set, 
		column \#A the number of attributes, numeric and nominal, in the format (numeric/nominal), and column IR the imbalance ratio (the number of instances of the majority class per instance of the minority class).}
	\label{tab:data-keel} \centering\scriptsize
	\begin{tabular}{@{}c@{\hspace{2mm}}c@{\hspace{2mm}}c@{}}
	\begin{tabular}{@{}l@{}rrr}
		\toprule
		Data set & \#E & \#A & IR \\
		\midrule
		abalone19 & 4174 & (7/1) & 129.44\\
		abalone9-18 & 731 & (7/1) & 16.40\\
		cleveland-0\_vs\_4 & 177 & (13/0) & 12.62\\
		ecoli-0-1-3-7\_vs\_2-6 & 281 & (7/0) & 39.14\\
		ecoli-0-1-4-6\_vs\_5 & 280 & (6/0) & 13.00\\
		ecoli-0-1-4-7\_vs\_2-3-5-6 & 336 & (7/0) & 10.59\\
		ecoli-0-1-4-7\_vs\_5-6 & 332 & (6/0) & 12.28\\
		ecoli-0-1\_vs\_2-3-5 & 244 & (7/0) & 9.17\\
		ecoli-0-1\_vs\_5 & 240 & (6/0) & 11.00\\
		ecoli-0-2-3-4\_vs\_5 & 202 & (7/0) & 9.10\\
		ecoli-0-2-6-7\_vs\_3-5 & 224 & (7/0) & 9.18\\
		ecoli-0-3-4-6\_vs\_5 & 205 & (7/0) & 9.25\\
		ecoli-0-3-4-7\_vs\_5-6 & 257 & (7/0) & 9.28\\
		ecoli-0-3-4\_vs\_5 & 200 & (7/0) & 9.00\\
		ecoli-0-4-6\_vs\_5 & 203 & (6/0) & 9.15\\
		ecoli-0-6-7\_vs\_3-5 & 222 & (7/0) & 9.09\\
		ecoli-0-6-7\_vs\_5 & 220 & (6/0) & 10.00\\
		ecoli-0\_vs\_1 & 220 & (7/0) & 1.86\\
		ecoli1 & 336 & (7/0) & 3.36\\
		ecoli2 & 336 & (7/0) & 5.46\\
		ecoli3 & 336 & (7/0) & 8.60\\
		ecoli4 & 336 & (7/0) & 15.80\\
		\bottomrule
	\end{tabular} 
	&
	\begin{tabular}{@{}l@{}rrr}
		\toprule
		Data set & \#E & \#A & IR \\
		\midrule
		glass-0-1-2-3\_vs\_4-5-6 & 214 & (9/0) & 3.20\\
		glass-0-1-4-6\_vs\_2 & 205 & (9/0) & 11.06\\
		glass-0-1-5\_vs\_2 & 172 & (9/0) & 9.12\\
		glass-0-1-6\_vs\_2 & 192 & (9/0) & 10.29\\
		glass-0-1-6\_vs\_5 & 184 & (9/0) & 19.44\\
		glass-0-4\_vs\_5 & 92 & (9/0) & 9.22\\
		glass-0-6\_vs\_5 & 108 & (9/0) & 11.00\\
		glass0 & 214 & (9/0) & 2.06\\
		glass1 & 214 & (9/0) & 1.82\\
		glass2 & 214 & (9/0) & 11.59\\
		glass4 & 214 & (9/0) & 15.46\\
		glass5 & 214 & (9/0) & 22.78\\
		glass6 & 214 & (9/0) & 6.38\\
		haberman & 306 & (3/0) & 2.78\\
		iris0 & 150 & (4/0) & 2.00\\
		led7digit-0-2-4-5-6-7-8-9\_vs\_1 & 443 & (7/0) & 10.97\\
		new-thyroid1 & 215 & (5/0) & 5.14\\
		new-thyroid2 & 215 & (5/0) & 5.14\\
		page-blocks-1-3\_vs\_4 & 472 & (10/0) & 15.86\\
		page-blocks0 & 5473 & (10/0) & 8.77\\
		pima & 768 & (8/0) & 1.87\\
		segment0 & 2310 & (19/0) & 6.00\\
		\bottomrule
	\end{tabular} 
	&
	\begin{tabular}{lrrr}
		\toprule
		Data set & \#E & \#A & IR \\
		\midrule
		shuttle-c0-vs-c4 & 1829 & (9/0) & 13.87\\
		shuttle-c2-vs-c4 & 129 & (9/0) & 20.50\\
		vehicle0 & 846 & (18/0) & 3.25\\
		vehicle1 & 846 & (18/0) & 2.90\\
		vehicle2 & 846 & (18/0) & 2.88\\
		vehicle3 & 846 & (18/0) & 2.99\\
		vowel0 & 988 & (13/0) & 9.98\\
		wisconsin & 683 & (9/0) & 1.86\\
		yeast-0-2-5-6\_vs\_3-7-8-9 & 1004 & (8/0) & 9.14\\
		yeast-0-2-5-7-9\_vs\_3-6-8 & 1004 & (8/0) & 9.14\\
		yeast-0-3-5-9\_vs\_7-8 & 506 & (8/0) & 9.12\\
		yeast-0-5-6-7-9\_vs\_4 & 528 & (8/0) & 9.35\\
		yeast-1-2-8-9\_vs\_7 & 947 & (8/0) & 30.57\\
		yeast-1-4-5-8\_vs\_7 & 693 & (8/0) & 22.10\\
		yeast-1\_vs\_7 & 459 & (7/0) & 14.30\\
		yeast-2\_vs\_4 & 514 & (8/0) & 9.08\\
		yeast-2\_vs\_8 & 482 & (8/0) & 23.10\\
		yeast1 & 1484 & (8/0) & 2.46\\
		yeast3 & 1484 & (8/0) & 8.10\\
		yeast4 & 1484 & (8/0) & 28.10\\
		yeast5 & 1484 & (8/0) & 32.73\\
		yeast6 & 1484 & (8/0) & 41.40\\
		\bottomrule
	\end{tabular}
\end{tabular}
\end{table*}

\subsection{Instance classification methods}


The two methods used as a baseline were the nearest neighbour classifier trained over the whole data set (1-NN, Method \#~1), and Bagging of 1-NN (BAG1NN, Method \#~2). Neither of these methods are adapted in any way to imbalanced data.

The list of instance selection methods we used is given below. The methods are grouped according to the level of randomness involved in the selection of instances: first are listed methods that are completely random; then methods with some degree of randomness; and finally, completely deterministic methods.

\medskip\noindent
{\em (3) RUS.} Random undersampling~\cite{barandela2003}: Let $N_+$ be the number of instances in the minority class. RUS balances the class distribution by drawing a random sample of size  $N_+$ from the majority class.

\medskip\noindent
{\em (4) ERUS.} Ensemble of RUS: ensemble of RUS combined by majority vote.

\medskip\noindent
{\em (5) RUSBOOST.} Boosting of RUS~\cite{seiffert2010}: an AdaBoost.M2~\cite{AdaBoostFreund1997} variant that performs re-weighting and RUS in each iteration.

\medskip\noindent
{\em (6) EUSBOOST}~\cite{galar2013}: AdaBoost-like ensemble of EUS (see next).

 \medskip\noindent
{\em (7) EUS.} Evolutionary undersampling~\cite{garcia2009taxonomy}: a version of a genetic algorithm which directly optimises GM. Only instances from the majority class are encoded as the chromosome. A taxonomy of evolutionary undersampling methods and an experimental study are presented in~\cite{garcia2009taxonomy}; we used the variant EBUS-MS-GM recommended by the authors.

\medskip\noindent
{\em (8) PSO.} Particle-swarm optimisation (PSO)~\cite{yang2009}: a particle-swarm algorithm maximising a combination of the most common metrics from imbalanced classes: the area under the ROC curve (AUC), the F-measure and GM. 
 
\medskip\noindent
{\em (9) TL.} Tomek links~\cite{tomek1976}: an adaptation of the Tomek links rule. Tomek links are pairs of instances from different classes such that the instances are each other's nearest neighbour. In the original rule, to clean the border from suspected noise, both instances in the Tomek link are removed. In the adapted version for imbalanced classes, only instances of the majority class which participate in Tomek links are removed.
  
\medskip\noindent
{\em (10) OSS.} One-sided selection~\cite{kubat1997}: condensing followed by TL. The condensing is done by a modification of Hart's Condensed Nearest Neighbour (CNN)~\cite{hart1968} for imbalanced classes~\cite{kubat1997}. The modified algorithm starts with all instances from the minority class plus a random majority instance. Majority instances are added, one at a time, only if misclassified by the current reference set.
 
\medskip\noindent
{\em (11) TL+CNN.} Tomek links + CNN (the modification for imbalanced data)~\cite{batista2004}\footnote{We noticed that, while the original OSS is defined by Kubat in~\cite{kubat1997} as CNN followed by TL, later on, Batista~\cite{batista2004} defined it in reverse order and also independently proposed an equivalent to Kubat's OSS. This misunderstanding has spread in subsequent works. However, we have maintained the original name OSS for CNN+TL, as used~\cite{kubat1997}, and we use TL+CNN for Batista et al.'s method~\cite{batista2004}.}.

\medskip\noindent
{\em (12) NCL.} Neighbourhood cleaning rule~\cite{laurikkala2001}: an adaptation of Wilson's Editing rule (ENN)~\cite{wilson1972} to the case of imbalanced learning. ENN marks and subsequently removes all instances misclassified by their 3 nearest neighbours. NCL does this only for instances from the majority class. If an instance from the minority class is misclassified by its 3 nearest neighbours, the instances voting for the majority class are marked for removal instead.

%

\medskip
The experiments were performed in Weka 3.7.11~\cite{wekaBookWitten2011}. The algorithms used directly from the Weka collection were 1NN, BAG1NN, and RUS\footnote{The random selection was performed by using the SpreadSubsample instance supervised filter.}. We created our own, not optimised, implementation of OSS, NCL, TL, TL+CNN, ERUS and RUSBOOST. We used publicly available code for EUSBOOST, EUS\footnote{Available in the KEEL GitHub repository: \url{https://github.com/SCI2SUGR/KEEL}}, and PSO\footnote{Available in Google code: \url{https://code.google.com/archive/p/imbalanced-data-sampling/}}, adapted to Weka by means of a wrapper. The ensemble size for the boosting methods (EUSBOOST and RUSBOOST) was set to 10, and for the other ensemble methods (BAG1NN and ERUS), to 100. The parameters of the preprogrammed methods were not changed from the default values, and the parameters for our implementations were the ones recommended in the original studies.

\subsection{Results and discussion}

As a result of the 5-times 2-fold cross-validation, each of the 66 data sets generates 10 GM values, giving a total of 660 values. For each data set, we used the same cross-validation folds for all the tested algorithms, meaning we can use the paired values of GM to determine which of a group of competing methods ``wins'' (has the largest GM value) for a given data set and fold. 




\begin{figure*}[htb]
\centering
\includegraphics[width=0.9\linewidth]{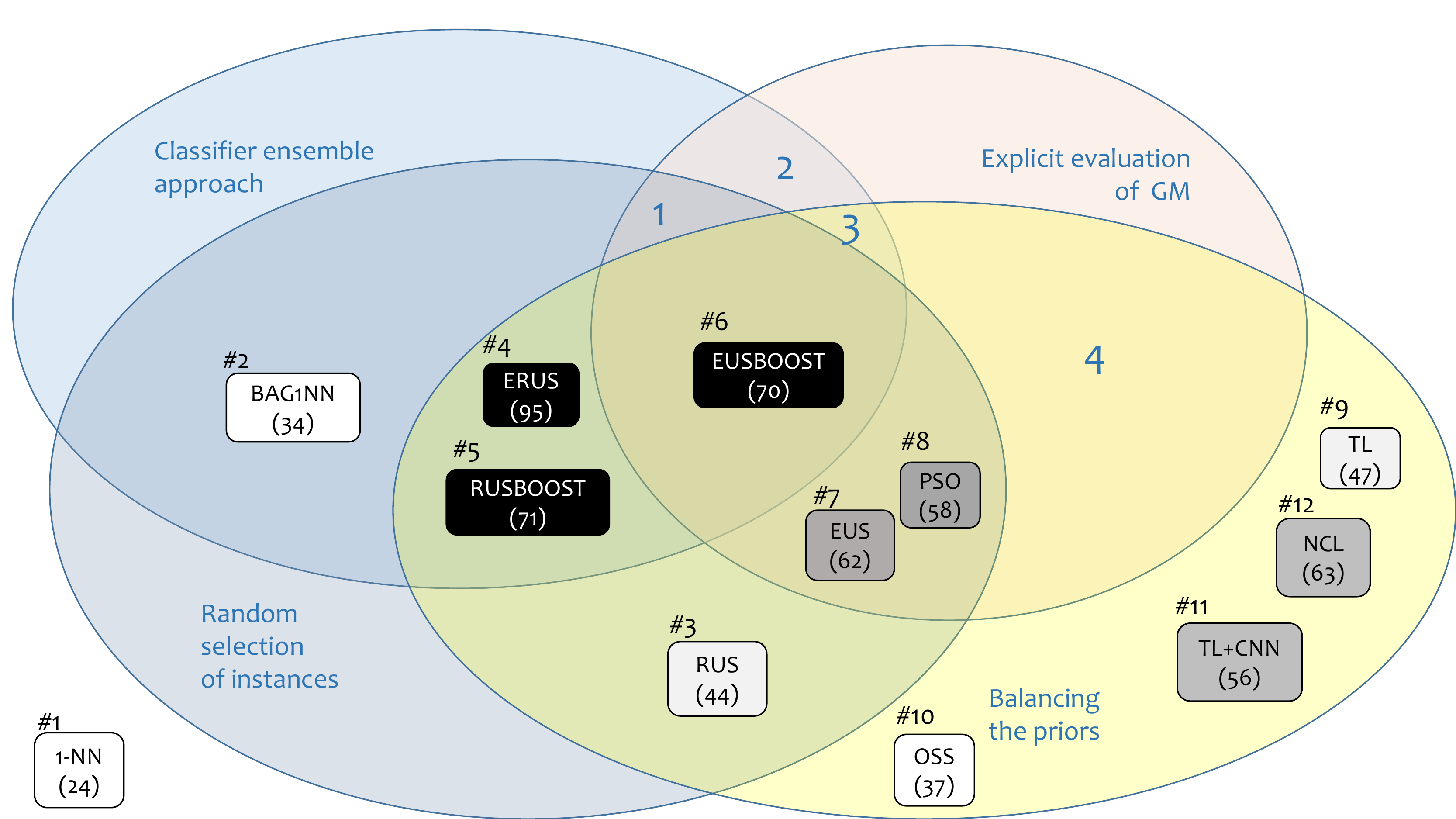}
\caption{A Venn diagram of the properties of the instance selection methods. 
The number in the brackets is the number of wins over all other methods out of the 660 comparisons. The shading of each method's box reflects this number of wins. The most successful methods in this experiment are shaded in black and the least successful, in white.}
\label{MethodsDiagram}
\end{figure*}

The main result of our experiment is shown in Figure \ref{MethodsDiagram} as a Venn diagram. The 12 imbalance classification methods are positioned within the ellipses representing the properties of interest. Each method is shown as round-cornered box. The number of the method is shown above the box. Underneath the method's name in the box, we give the total number of wins for this method out of the 660 comparisons. The tied wins were split among the winning methods: if methods A, B, and C won a comparison, one third was added to the winning count of each method. The figure shows the number of wins calculated in this way, rounded to the nearest integer.  The boxes of the methods are coloured in different shades of grey reflecting the number of wins of the method. The most successful methods are shown in black, and the least successful in white.

We look at the performance of the methods in relationship to: whether the method explicitly evaluates the GM measure; whether it balances the prior probabilities of the classes; whether it makes use of random selection; and, finally, we add another property:  whether the method uses an ensemble approach. Note that possession of the second and third properties should not be viewed as a perfectly binary question.  In particular, methods 9--12 only partially reduce the size of the larger class, so are less strongly balancing than the methods which enforce equal class sizes, and methods 6--8 have the randomness of their selections tempered by an evolutionary process. 

The full set of results (a table with GM values, of size $12\times 660$) is provided in the supplementary material. Table~\ref{tab:stats} shows the p-values for the sign test between the GM values for the 12 methods. Each pair of methods generates one entry in the table. The statistically significant differences where the method of the row is better than the method of the column at level 0.05 are marked with boxes. Bonferroni correction for multiple comparisons has been applied.

\begin{table*}
\caption{P-values for the sign test between the GM values for the 12 methods. The statistically significant differences where the method of the row is better than the method of the column at level 0.05 are marked with boxes. Bonferroni correction for multiple comparisons has been applied.}
\label{tab:stats}
\centering
\begin{tabular}{rrrrrrrrrrrrr}
&(1) & (2) & (3) & (4) & (5) & (6) & (7) & (8) & (9) & (10) & (11) & (12) \\
\hline
1-NN (1)& 1.000 & 1.000 & 1.000 & 1.000 & 1.000 & 1.000 & 1.000 & 1.000 & 1.000 & 1.000 & 1.000 & 1.000 \\
BAG1NN (2)& \fbox{0.000} & 1.000 & 1.000 & 1.000 & 1.000 & 1.000 & 1.000 & 1.000 & 1.000 & 1.000 & 1.000 & 1.000 \\
RUS (3)& \fbox{0.000} & \fbox{0.000} & 1.000 & 1.000 & 1.000 & 1.000 & 1.000 & 0.926 & 0.108 & \fbox{0.000} & 0.437 & 0.868 \\
ERUS (4)& \fbox{0.000} & \fbox{0.000} & \fbox{0.000} & 1.000 & 0.340 & 0.452 & 0.034 & \fbox{0.000} & \fbox{0.000} & \fbox{0.000} & \fbox{0.000} & \fbox{0.000} \\
RUSBOOST (5)& \fbox{0.000} & \fbox{0.000} & \fbox{0.000} & 0.690 & 1.000 & 0.989 & 0.435 & \fbox{0.000} & \fbox{0.000} & \fbox{0.000} & \fbox{0.000} & 0.009 \\
EUSBOOST (6)& \fbox{0.000} & \fbox{0.000} & \fbox{0.000} & 0.580 & 0.013 & 1.000 & \fbox{0.000} & \fbox{0.000} & \fbox{0.000} & \fbox{0.000} & \fbox{0.000} & 0.003 \\
EUS (7)& \fbox{0.000} & \fbox{0.000} & \fbox{0.000} & 0.971 & 0.597 & 1.000 & 1.000 & \fbox{0.000} & \fbox{0.000} & \fbox{0.000} & \fbox{0.000} & 0.038 \\
PSO (8)& \fbox{0.000} & \fbox{0.000} & 0.086 & 1.000 & 1.000 & 1.000 & 1.000 & 1.000 & 0.028 & \fbox{0.000} & 0.007 & 0.532 \\
TL (9)& \fbox{0.000} & \fbox{0.000} & 0.906 & 1.000 & 1.000 & 1.000 & 1.000 & 0.977 & 1.000 & 0.112 & 0.213 & 0.906 \\
OSS (10)& \fbox{0.000} & \fbox{0.000} & 1.000 & 1.000 & 1.000 & 1.000 & 1.000 & 1.000 & 0.902 & 1.000 & 1.000 & 1.000 \\
TL+CNN (11)& \fbox{0.000} & \fbox{0.000} & 0.594 & 1.000 & 1.000 & 1.000 & 1.000 & 0.994 & 0.810 & \fbox{0.000} & 1.000 & 1.000 \\
NCL (12)& \fbox{0.000} & \fbox{0.000} & 0.150 & 1.000 & 0.993 & 0.997 & 0.968 & 0.500 & 0.110 & \fbox{0.000} & \fbox{0.000} & 1.000 \\
\hline
\end{tabular}
\end{table*}

Here we consider the findings from the experiment with reference to the three points listed at the beginning of this section. 

\medskip
\noindent
{\em 1) Instance selection improves GM.} This is evident from the low number of wins for the 1-NN classifier (24) compared to all other methods (boxes in Figure~\ref{MethodsDiagram}), and from the first column in Table~\ref{tab:stats}, showing that all methods applying instance selection are significantly better than 1-NN.

\medskip
\noindent
{\em 2) Random prototype selection is generally better than systematic selection.} 
The three most successful algorithms are of this type. Excluding 1-NN where no selection is carried out, random selection is represented through BAG1NN~(34), ERUS~(95), EUSBOOST~(70), RUSBOOST~(71), PSO~(58), EBUS~(62), and RUS~(44), averaging at 62 wins. The alternative group contains OSS~(37), TL~(47), TL+CNN~(56), and NCL~(63), with average 50.8 wins. 

Table~\ref{tab:stats} offers another quantification method for this issue. Consider the group of 7 methods with random selection $\{2, 3, 4, 5, 6, 7, 8\}$ and the alternative group of 4 methods $\{9, 10, 11, 12\}$. Thus, there are $7\times 4 = 28$ pairwise comparisons $A$ versus $B$, where $A$ is a method from the first group, and $B$, from the second group. In these 24 comparisons, The number of significant wins of the methods in the first group (random selection), counted from Table~\ref{tab:stats} is 15, while the number of significant wins for the second group (non-random selection) is 4. For the remaining 5 comparisons the differences were not found to be significant.

Both arguments above lend support to our theoretical arguments that random prototype selection is preferable in view of the non-monotonicity of GM, and the fact that the importance of a prototype is determined only in reference to its neighbours. 

\medskip
\noindent
{\em 3) Is optimising GM explicitly more successful than balancing the prior probabilities alone or not accounting for the imbalance at all?}
The Venn diagram in Figure~\ref{MethodsDiagram} does not show clear evidence either way. The group of explicitly-optimising algorithms is smaller, yet all three perform better than the majority of the larger group.  Table~\ref{tab:stats} gives a finer-grained quantification of the difference between the two groups. This time, the groups are: methods $\{6, 7, 8\}$ for the explicit GM optimisation, and methods $\{2, 3, 4, 5, 9, 10, 11, 12\}$ for the other methods, giving again $3 \times 8 = 24$ pairwise comparisons. Out of these, the number of statistically significant wins for the first group (from Table~\ref{tab:stats}) is 12, and for the second (larger) group,~4.

Based on the above, we suggest that explicitly-optimising algorithms may be the more fruitful area for further development, despite the fact that the current best methods are not of this type.

Finally, we observed that Ensemble methods for instance selection are better than single classifiers, when using NN methods for imbalanced data. Classifier ensembles are known to be generally better than single classifiers in the traditional cases of balanced data~\cite{kuncheva2014}. However, it is often said that ensembling of 1-NN classifiers does not work very well due to insufficient diversity. 
We did see poor performance in our experiment from the classical bagging 1-NN ensemble (BAG1NN). However, the other ensemble instance selection methods (those which are combined with an imbalance-specific approach such as balancing the priors or explicit evaluation of GM) score above the non-ensemble methods. (Compare the black boxes in Figure~\ref{MethodsDiagram} with the non-black boxes, excluding BAG1NN.) The average win score of the ensemble methods (BAG1NN~(34), ERUS~(95), RUSBOOST~(71), and EUSBOOST~(70)) is 67.5 whereas the average win of the other methods (excluding 1-NN, the worst case) is 52.4. 

Our categorisation of the existing methods according to four properties suggests that there are unexplored spaces for designing potentially-successful methods for imbalanced data. Note the areas in Figure~\ref{MethodsDiagram} marked with numbers 1, 2, 3, and 4. To the best of our knowledge, there are no widely-used methods with the particular combination of properties which correspond to these regions. It may be interesting to develop algorithms in these areas and examine their performance in relation to the ones in this study.  In particular, in view of the reasons discussed above for considering methods which explicitly optimize GM but do not explicitly balance the priors, and in view of the superior performance of editing methods which make use of random sampling, the possibility of developing algorithms in the area marked 1 seems most promising.

\section{Conclusions}
In view of the importance of the classification of imbalanced data for real-life problems, we 
sought to give some theoretical argument in support of the heuristics and techniques for data manipulation, which seemed to be lacking so far.
Hence, this paper considers the geometric mean of the true positive rate and the true negative rate (GM) in two-class imbalanced problems from a general instance-selection perspective. An experiment with 12 state-of-the-art classification methods and 66 benchmark data sets verifies our theoretical arguments and hypotheses. 

We prove that the performance of 1-NN with respect to GM may be improved by instance selection. This is reinforced by the numerical example we give, and also by the results of our experimental study.

Next we argue that systematic instance selection may be inferior to random instance selection in relation to GM, which lends support to one of the most successful approaches to handling imbalanced data. This is also confirmed by our experimental findings.

We also prove by counter-examples that the Bayes classifier as well as the ``Balanced Bayes'' classifier (which assigns an object to the class with the greater class-conditional probability density at each point, not scaled by the prior probability), are not in general GM-optimal.  This suggests that designing algorithms to balance the class priors for imbalanced data problems is not necessarily GM-optimal. However, while optimality is not guaranteed, balancing the priors is still one of the most successful heuristics. Our empirical findings indicate that explicit maximisation of GM may be a promising ground for future research.


We show with the help of a Venn diagram that there may be unexplored possibilities for devising successful instance selection methods for imbalanced data which optimise GM explicitly.

Finally, we list some limitations of our study:
\begin{itemize}
\item We consider instance selection for the {\em single} nearest neighbour classifier, the 1-NN. Better results may be obtained using $k$-NN.
\item Euclidean distance was used throughout for finding the prototypes and forming the Voronoi cells. We did not consider alternative metrics or distance metric adaptations, local or global.
\item The prototypes forming the reference set were assumed to be elements of the training set, and therefore fixed in the space (the instance-selection approach). It can be proved that better results are possible if we allow the prototypes' locations to be tuned.
\item The performance measure used throughout this study was the geometric mean of the true positive and true negative rate. While the F-measure and the AUC are related to GM, our results do not directly extend to these measures.
\item We considered two-class problems because GM is defined for two classes. However, there are ways of  extending our results to multi-class problems.
\end{itemize}

\section*{Acknowledgment}
This work was done under project RPG-2015-188 funded by The Leverhulme Trust, UK; the project TIN2015-67534-P funded by the \emph{Ministerio de Economía y Competitividad} of the Spanish Government and the BU085P17 funded by the \emph{Junta de Castilla y León}.


\begin{thebibliography}{10}
\expandafter\ifx\csname url\endcsname\relax
  \def\url#1{\texttt{#1}}\fi
\expandafter\ifx\csname urlprefix\endcsname\relax\def\urlprefix{URL }\fi
\expandafter\ifx\csname href\endcsname\relax
  \def\href#1#2{#2} \def\path#1{#1}\fi

\bibitem{chawla2004}
N.~Chawla, N.~Japkowicz, A.~Kotcz, Editorial: special issue on learning from
  imbalanced data sets, ACM SIGKDD Explorations Newsletter 6~(1) (2004) 1--6.

\bibitem{Pedrajas2012}
N.~Garc\'{\i}a-Pedrajas, J.~P{\'e}rez-Rodr\'{\i}guez, M.~D.
  Garc\'{\i}a-Pedrajas, D.~Ortiz-Boyer, C.~Fyfe, Class imbalance methods for
  translation initiation site recognition in {DNA} sequences, Knowledge-Based
  Systems 25~(1) (2012) 22--34.

\bibitem{krawczyk2016}
B.~Krawczyk, M.~Galar, {\L}.~Jele{\'n}, F.~Herrera, Evolutionary undersampling
  boosting for imbalanced classification of breast cancer malignancy, Applied
  Soft Computing 38 (2016) 714--726.

\bibitem{eskildsen2014}
S.~F. Eskildsen, P.~Coup{\'e}, V.~Fonov, D.~L. Collins, Detecting alzheimer's
  disease by morphological mri using hippocampal grading and cortical
  thickness, in: Proc MICCAI Workshop Challenge on Computer-Aided Diagnosis of
  Dementia Based on Structural MRI Data, 2014, pp. 38--47.

\bibitem{cieslak2006}
D.~A. Cieslak, N.~V. Chawla, A.~Striegel, Combating imbalance in network
  intrusion datasets., in: GrC, 2006, pp. 732--737.

\bibitem{tesfahun2013}
A.~Tesfahun, D.~L. Bhaskari, Intrusion detection using random forests
  classifier with smote and feature reduction, in: Cloud \& Ubiquitous
  Computing \& Emerging Technologies (CUBE), 2013 International Conference on,
  IEEE, 2013, pp. 127--132.

\bibitem{dal2014}
A.~Dal~Pozzolo, O.~Caelen, Y.-A. Le~Borgne, S.~Waterschoot, G.~Bontempi,
  Learned lessons in credit card fraud detection from a practitioner
  perspective, Expert systems with applications 41~(10) (2014) 4915--4928.

\bibitem{phua2004}
C.~Phua, D.~Alahakoon, V.~Lee, Minority report in fraud detection:
  classification of skewed data, Acm sigkdd explorations newsletter 6~(1)
  (2004) 50--59.

\bibitem{sanz2015}
J.~A. Sanz, D.~Bernardo, F.~Herrera, H.~Bustince, H.~Hagras, A compact
  evolutionary interval-valued fuzzy rule-based classification system for the
  modeling and prediction of real-world financial applications with imbalanced
  data, Fuzzy Systems, IEEE Transactions on 23~(4) (2015) 973--990.

\bibitem{drown2009}
D.~J. Drown, T.~M. Khoshgoftaar, N.~Seliya, Evolutionary sampling and software
  quality modeling of high-assurance systems, Systems, Man and Cybernetics,
  Part A: Systems and Humans, IEEE Transactions on 39~(5) (2009) 1097--1107.

\bibitem{seiffert2009}
C.~Seiffert, T.~M. Khoshgoftaar, J.~Van~Hulse, Improving software-quality
  predictions with data sampling and boosting, Systems, Man and Cybernetics,
  Part A: Systems and Humans, IEEE Transactions on 39~(6) (2009) 1283--1294.

\bibitem{sun2012}
Z.~Sun, Q.~Song, X.~Zhu, Using coding-based ensemble learning to improve
  software defect prediction, Systems, Man, and Cybernetics, Part C:
  Applications and Reviews, IEEE Transactions on 42~(6) (2012) 1806--1817.

\bibitem{zheng2015}
B.~Zheng, S.~W. Myint, P.~S. Thenkabail, R.~M. Aggarwal, A support vector
  machine to identify irrigated crop types using time-series landsat ndvi data,
  International Journal of Applied Earth Observation and Geoinformation 34
  (2015) 103--112.

\bibitem{Visa2005}
S.~Visa, A.~Ralescu, Issues in mining imbalanced data sets-a review paper, in:
  Proceedings of the sixteen midwest artificial intelligence and cognitive
  science conference, Vol. 2005, sn, 2005, pp. 67--73.

\bibitem{Galar2012}
M.~Galar, A.~Fernandez, E.~Barrenechea, H.~Bustince, F.~Herrera, A review on
  ensembles for the class imbalance problem: Bagging-, boosting-, and
  hybrid-based approaches, Systems, Man, and Cybernetics, Part C: Applications
  and Reviews, IEEE Transactions on 42~(4) (2012) 463 --484.

\bibitem{diez2015div}
J.~F. D{\'\i}ez-Pastor, J.~J. Rodr{\'\i}guez, C.~I. Garc{\'\i}a-Osorio, L.~I.
  Kuncheva, Diversity techniques improve the performance of the best imbalance
  learning ensembles, Information Sciences 325 (2015) 98--117.

\bibitem{Fix52}
E.~Fix, J.~L. Hodges, Discriminatory analysis : Non parametric discrimination :
  Small sample performance, Tech. Rep. Project 21 - 49 - 004 (11), USAF School
  of Aviation Medicine, Randolph Field,Texas (1952).

\bibitem{Cover67}
T.~Cover, P.~Hart, Nearest neighbor pattern classification, IEEE Trans.
  Information Theory 13~(1) (1967) 21--27.

\bibitem{Wu08}
X.~Wu, V.~Kumar, J.~R. Quinlan, J.~Ghosh, Q.~Yang, H.~Motoda, G.~J. McLachlan,
  A.~Ng, B.~Liu, P.~S. Yu, Z.-H. Zhou, M.~Steinbach, D.~J. Hand, D.~Steinberg,
  Top 10 algorithms in data mining, Knowledge and Information Systems 14~(1)
  (2008) 1--37.

\bibitem{DROP3Wilson2000}
D.~R. Wilson, T.~R. Martinez, Reduction techniques for instance-based learning
  algorithms, Machine Learning 38~(3) (2000) 257--286.
\newblock \href {http://dx.doi.org/10.1023/A:1007626913721}
  {\path{doi:10.1023/A:1007626913721}}.

\bibitem{Dasarathy90}
B.~V. Dasarathy, Nearest Neighbor {(NN)} Norms: {NN} Pattern Classification
  Techniques, IEEE Computer Society Press, Los Alamitos, California, 1990.

\bibitem{garciaSalvador2012}
S.~Garcia, J.~Derrac, J.~Cano, F.~Herrera, Prototype selection for nearest
  neighbor classification: Taxonomy and empirical study, Pattern Analysis and
  Machine Intelligence, IEEE Transactions on 34~(3) (2012) 417--435.
\newblock \href {http://dx.doi.org/10.1109/TPAMI.2011.142}
  {\path{doi:10.1109/TPAMI.2011.142}}.

\bibitem{Triguero12}
I.~Triguero, J.~Derrac, S.~Garc\'ia, F.~Herrera, A taxonomy and experimental
  study on prototype generation for nearest neighbor classification, {IEEE}
  Transactions on Systems, Man, and Cybernetics, Part C: Applications and
  Reviews 42~(1) (2012) 86--100.

\bibitem{batista2004}
G.~E. A. P.~A. Batista, R.~C. Prati, M.~C. Monard, A study of the behavior of
  several methods for balancing machine learning training data, SIGKDD Explor.
  Newsl. 6~(1) (2004) 20--29.
\newblock \href {http://dx.doi.org/10.1145/1007730.1007735}
  {\path{doi:10.1145/1007730.1007735}}.

\bibitem{galar2013}
M.~Galar, A.~Fern{\'a}ndez, E.~Barrenechea, F.~Herrera, Eusboost: enhancing
  ensembles for highly imbalanced data-sets by evolutionary undersampling,
  Pattern Recognition 46~(12) (2013) 3460--3471.

\bibitem{Akbani2004}
R.~Akbani, S.~Kwek, N.~Japkowicz, Applying support vector machines to
  imbalanced datasets, in: J.-F. Boulicaut, F.~Esposito, F.~Giannotti,
  D.~Pedreschi (Eds.), Machine Learning: ECML 2004: 15th European Conference on
  Machine Learning, Pisa, Italy, September 20-24, 2004. Proceedings, Springer
  Berlin Heidelberg, Berlin, Heidelberg, 2004, pp. 39--50.

\bibitem{Batuwita2010}
R.~Batuwita, V.~Palade, {FSVM-CIL}: Fuzzy support vector machines for class
  imbalance learning, IEEE Transactions on Fuzzy Systems 18~(3) (2010)
  558--571.
\newblock \href {http://dx.doi.org/10.1109/TFUZZ.2010.2042721}
  {\path{doi:10.1109/TFUZZ.2010.2042721}}.

\bibitem{mani2003}
J.~Zhang, I.~Mani, k{NN} approach to unbalanced data distributions: a case
  study involving information extraction, in: Proceedings of The Twentieth
  International Conference on Machine Learning (ICML-2003), Workshop on
  Learning from Imbalanced Data Sets, 2003.

\bibitem{barandela2003a}
R.~Barandela, J.~S\'anchez, V.~Garc\'ia, E.~Rangel, Strategies for learning in
  class imbalance problems, Pattern Recognition 36~(3) (2003) 849--851.

\bibitem{lopez2014}
V.~L\'opez, I.~Triguero, C.~J. Carmona, S.~Garc\'ia, F.~Herrera, Addressing
  imbalanced classification with instance generation techniques: {IPADE-ID},
  Neurocomputing 126 (2014) 15 -- 28.

\bibitem{saez2016}
J.~A. S\'aez, B.~Krawczyk, M.~Wo\'zniak, Analyzing the oversampling of
  different classes and types of examples in multi-class imbalanced datasets,
  Pattern Recognition (2016) --\href
  {http://dx.doi.org/10.1016/j.patcog.2016.03.012}
  {\path{doi:10.1016/j.patcog.2016.03.012}}.

\bibitem{Pudil94}
P.~Pudil, J.~Novovi{\v c}ov\'a, J.~Kittler, Floating search methods in feature
  selection, Pattern Recognition Letters 15~(11) (1994) 1119 -- 1125.
\newblock \href {http://dx.doi.org/10.1016/0167-8655(94)90127-9}
  {\path{doi:10.1016/0167-8655(94)90127-9}}.

\bibitem{Jain97a}
A.~K. Jain, D.~Zongker, Feature selection: evaluation, application, and small
  sample performance, IEEE Transactions on Pattern Analysis and Machine
  Intelligence 19~(2) (1997) 153--158.
\newblock \href {http://dx.doi.org/10.1109/34.574797}
  {\path{doi:10.1109/34.574797}}.

\bibitem{Saeys07}
Y.~Saeys, I.~Inza, P.~L. {n}aga, A review of feature selection techniques in
  bioinformatics, Bioinformatics 23~(19) (2007) 2507--2517.

\bibitem{dietterich1998}
T.~Dietterich, {Approximate statistical tests for comparing supervised
  classification learning algorithms}, Neural computation 10~(7) (1998)
  1895--1923.

\bibitem{diez2015rb}
J.~F. D{\'\i}ez-Pastor, J.~J. Rodr{\'\i}guez, C.~Garc{\'\i}a-Osorio, L.~I.
  Kuncheva, Random balance: ensembles of variable priors classifiers for
  imbalanced data, Knowledge-Based Systems 85 (2015) 96--111.

\bibitem{Alcala2011}
J.~Alcala-Fdez, A.~Fern\'andez, J.~Luengo, J.~Derrac, S.~Garc\'ia,
  L.~S\'anchez, F.~Herrera, {KEEL} data-mining software tool: Data set
  repository and integration of algorithms and experimental analysis framework,
  Journal of Multiple-Valued Logic and Soft Computing 17~(2-3) (2011) 255--287.

\bibitem{barandela2003}
R.~Barandela, R.~Valdovinos, J.~S{\'a}nchez, New applications of ensembles of
  classifiers, Pattern Analysis \& Applications 6~(3) (2003) 245--256.

\bibitem{seiffert2010}
C.~Seiffert, T.~Khoshgoftaar, J.~Van~Hulse, A.~Napolitano, {RUSBoost}: A hybrid
  approach to alleviating class imbalance, Systems, Man and Cybernetics, Part
  A: Systems and Humans, IEEE Transactions on 40~(1) (2010) 185--197.

\bibitem{AdaBoostFreund1997}
Y.~Freund, R.~E. Schapire, A decision-theoretic generalization of on-line
  learning and an application to boosting, Journal of Computer and System
  Sciences 55~(1) (1997) 119 -- 139.
\newblock \href {http://dx.doi.org/http://dx.doi.org/10.1006/jcss.1997.1504}
  {\path{doi:http://dx.doi.org/10.1006/jcss.1997.1504}}.

\bibitem{garcia2009taxonomy}
S.~García, F.~Herrera, Evolutionary undersampling for classification with
  imbalanced datasets: Proposals and taxonomy, Evolutionary Computation 17~(3)
  (2009) 275--306.
\newblock \href {http://dx.doi.org/10.1162/evco.2009.17.3.275}
  {\path{doi:10.1162/evco.2009.17.3.275}}.

\bibitem{yang2009}
P.~Yang, L.~Xu, B.~B. Zhou, Z.~Zhang, A.~Y. Zomaya, A particle swarm based
  hybrid system for imbalanced medical data sampling, BMC Genomics 10~(3)
  (2009) 1--14.
\newblock \href {http://dx.doi.org/10.1186/1471-2164-10-S3-S34}
  {\path{doi:10.1186/1471-2164-10-S3-S34}}.

\bibitem{tomek1976}
I.~Tomek, Two modifications of {CNN}, IEEE Trans. Systems, Man and Cybernetics
  6 (1976) 769--772.

\bibitem{kubat1997}
M.~Kubat, S.~Matwin, Addressing the curse of imbalanced training sets:
  One-sided selection, in: In Proceedings of the Fourteenth International
  Conference on Machine Learning, Morgan Kaufmann, 1997, pp. 179--186.

\bibitem{hart1968}
P.~Hart, The condensed nearest neighbor rule (corresp.), Information Theory,
  IEEE Transactions on 14~(3) (1968) 515 -- 516.

\bibitem{laurikkala2001}
J.~Laurikkala, Improving identification of difficult small classes by balancing
  class distribution, in: S.~Quaglini, P.~Barahona, S.~Andreassen (Eds.),
  Artificial Intelligence in Medicine: 8th Conference on Artificial
  Intelligence in Medicine in Europe, AIME 2001 Cascais, Portugal, July 1--4,
  2001, Proceedings, Springer Berlin Heidelberg, Berlin, Heidelberg, 2001, pp.
  63--66.

\bibitem{wilson1972}
D.~L. Wilson, Asymptotic properties of nearest neighbor rules using edited
  data, Systems, Man and Cybernetics, IEEE Transactions on SMC-2~(3) (1972)
  408--421.
\newblock \href {http://dx.doi.org/10.1109/TSMC.1972.4309137}
  {\path{doi:10.1109/TSMC.1972.4309137}}.

\bibitem{wekaBookWitten2011}
I.~H. Witten, E.~Frank, M.~A. Hall, Data Mining: Practical Machine Learning
  Tools and Techniques, 3rd Edition, Morgan Kaufmann Publishers Inc., San
  Francisco, CA, USA, 2011.

\bibitem{kuncheva2014}
L.~Kuncheva, \href{https://books.google.co.uk/books?id=RtRLBAAAQBAJ}{Combining
  Pattern Classifiers: Methods and Algorithms}, Wiley, 2014.
\newline\urlprefix\url{https://books.google.co.uk/books?id=RtRLBAAAQBAJ}

\end{thebibliography}
\end{document}